\documentclass[11pt]{article}
\usepackage[T1]{fontenc}

\usepackage[margin=1in]{geometry}
\usepackage[colorlinks,linkcolor=blue,citecolor=blue]{hyperref}
\usepackage{amsmath,amsfonts,amssymb,amsthm}
\usepackage{algorithm,algorithmic,caption}
\usepackage{enumitem}
\usepackage[round]{natbib}

\title{Learning Linear-Quadratic Regulators \\ Efficiently with only $\sqrt{T}$ Regret}
\author{%
Alon Cohen%
\thanks{Technion---Israel Inst.~of Technology and Google Tel Aviv; \texttt{alon.cohen@campus.technion.ac.il}.}
\and
Tomer Koren%
\thanks{Google Brain, Mountain View; \texttt{tkoren@google.com}.}
\and 
Yishay Mansour%
\thanks{Tel-Aviv University and Google Tel Aviv; \texttt{mansour.yishay@gmail.com}.}
}

\usepackage{amsthm} % better to load after ams math
\usepackage{mathtools,thmtools}

\usepackage{enumitem}
\usepackage{bm}
\usepackage{xspace}
\usepackage{nicefrac}
\usepackage[capitalise]{cleveref}
\usepackage[usenames,dvipsnames]{xcolor}
\usepackage{dsfont}

\usepackage[textsize=footnotesize,textwidth=0.75in,
            color=cyan,bordercolor=white]{todonotes}
\presetkeys{todonotes}{size=\footnotesize\sffamily\bfseries\boldmath\color{white}}{}

\newcommand{\ignore}[1]{}

\declaretheoremstyle[
	    spaceabove=\topsep, 
	    spacebelow=\topsep, 
	    bodyfont=\normalfont\itshape,
    ]{theorem}

\declaretheorem[style=theorem,name=Theorem]{theorem}

\declaretheoremstyle[
	    spaceabove=\topsep, 
	    spacebelow=\topsep, 
	    bodyfont=\normalfont,
    ]{definition}
    
\declaretheorem[style=theorem,sibling=theorem,name=Lemma]{lemma}
\declaretheorem[style=theorem,sibling=theorem,name=Corollary]{corollary}

\declaretheorem[style=theorem,numbered=no,name=Theorem]{theorem*}
\declaretheorem[style=theorem,numbered=no,name=Lemma]{lemma*}
\declaretheorem[style=theorem,numbered=no,name=Corollary]{corollary*}
\declaretheorem[style=theorem,numbered=no,name=Proposition]{proposition*}
\declaretheorem[style=theorem,numbered=no,name=Claim]{claim*}
\declaretheorem[style=theorem,numbered=no,name=Fact]{fact*}
\declaretheorem[style=theorem,numbered=no,name=Observation]{observation*}
\declaretheorem[style=theorem,numbered=no,name=Conjecture]{conjecture*}

\declaretheorem[style=definition,sibling=theorem,name=Definition]{definition}

\declaretheorem[style=definition,numbered=no,name=Definition]{definition*}
\declaretheorem[style=definition,numbered=no,name=Remark]{remark*}
\declaretheorem[style=definition,numbered=no,name=Example]{example*}
\declaretheorem[style=definition,numbered=no,name=Question]{question*}

\DeclareMathAlphabet{\mathbfsf}{\encodingdefault}{\sfdefault}{bx}{n}

\DeclareMathOperator*{\trace}{Tr}

\let\Pr\relax
\DeclareMathOperator{\Pr}{\mathbb{P}}

\newcommand{\lr}[1]{\mathopen{}\left(#1\right)}
\newcommand{\Lr}[1]{\mathopen{}\big(#1\big)}
\newcommand{\LR}[1]{\mathopen{}\Big(#1\Big)}

\newcommand{\lrbra}[1]{\mathopen{}\left[#1\right]}
\newcommand{\Lrbra}[1]{\mathopen{}\big[#1\big]}

\newcommand{\norm}[1]{\|#1\|}

\newcommand{\LRnorm}[1]{\mathopen{}\Big\|#1\Big\|}

\newcommand{\set}[1]{\{#1\}}

\newcommand{\lrfloor}[1]{\mathopen{}\left\lfloor #1 \right\rfloor}

\newcommand{\ind}[1]{\mathbb{I}\set{#1}}

\newcommand{\wt}[1]{\smash{\widetilde{#1}}}

\renewcommand{\O}{O}

\newcommand{\tO}{\wt{\O}}

\newcommand{\E}{\mathbb{E}}

\newcommand{\poly}{\mathrm{poly}}

\newcommand{\tr}{^{\mkern-1.5mu\scriptstyle\mathsf{T}}}

\newcommand{\st}{\star}

\newcommand{\reals}{\mathbb{R}}

\newcommand{\half}{\frac{1}{2}}
\newcommand{\thalf}{\tfrac{1}{2}}

\let\oldtfrac\tfrac
\renewcommand{\tfrac}[2]{\smash{\oldtfrac{#1}{#2}}}

\let\nablaold\nabla
\renewcommand{\nabla}{\nablaold\mkern-1mu}

\crefname{ALC@unique}{line}{lines}

\DeclareMathOperator{\rank}{rank}
\newcommand{\frob}{\mathsf{F}}

\renewcommand{\LRnorm}[1]{\Bigg\|#1\Bigg\|}

\renewcommand{\LR}[1]{\bigg(#1\bigg)}

\newcommand{\calE}{\mathcal{E}}

\newcommand{\calF}{\mathcal{F}}
\newcommand{\calN}{\mathcal{N}}
\newcommand{\calK}{\mathcal{K}}
\newcommand{\bbR}{\mathbb{R}}
\newcommand{\bbE}{\mathbb{E}}
\newcommand{\bbS}{\mathbb{S}}
\newcommand{\spnorm}[1]{\norm{#1}}
\newcommand{\trnorm}[1]{\norm{#1}_{*}}

\newcommand{\IhatK}{\begin{psmallmatrix} I \\ K \end{psmallmatrix}}
\newcommand{\IhatKt}{\begin{psmallmatrix} I \\ K_t \end{psmallmatrix}}
\newcommand{\IhatKs}{\begin{psmallmatrix} I \\ K_s \end{psmallmatrix}}
\newcommand{\qrmatrix}{\begin{psmallmatrix} Q & 0 \\ 0 & R \end{psmallmatrix}}
\newcommand{\ws}{0}
\newcommand{\Ast}{A_\st}
\newcommand{\Bst}{B_\st}
\newcommand{\Kst}{K_\st}
\newcommand{\IKt}{\begin{psmallmatrix} I \\ K_t \end{psmallmatrix}}

\begin{document}
\maketitle

\begin{abstract}
We present the first computationally-efficient algorithm with $\tO(\sqrt{T})$ regret for learning in Linear Quadratic Control systems with unknown dynamics.
%
%Our methodology presents a semi-definite program whose solution can be related to optimizing linear-quadratic systems.
%
%quadratic regulator.
%In our model the linear dynamics are unknown, but the quadratic costs are given.
%
%Similar to other approaches, our algorithm balances exploration and exploitation by executing a sequence of ``optimistic'' policies.
%
%The novelty of our algorithm stems computing each policy in polynomial-time by solving a novel semi-definite program relaxation.
%
By that, we resolve an open question of \citet*{abbasi2011regret} and \citet*{dean2018regret}.
\end{abstract}

\section{Introduction}

% Reinforcement learning and control theory are deeply intertwined.
Optimal control theory dates back to the 1950s, and has been applied successfully to numerous real-world engineering problems \citep[e.g.,][]{bermudez1994state,chen2005optimal,lenhart2007optimal,geering2007optimal}. %,anita2011introduction}.
Classical results in control theory pertain to asymptotic convergence and stability of dynamical systems, and recently, there has been a renewed interest in such problems from a learning-theoretic perspective with a focus on finite-time convergence guarantees and computational tractability.
% In recent years, demand for solving challenging, complex control problems has bred a need for data-driven solutions.
%Naturally, such a controller's goal would be to learn the system parameters. Hence, the field of adaptive control has raised interest in the machine learning community as well. In particular, techniques from reinforcement learning have been used extensively in practical applications of adaptive control
%\cite{bradtke1994adaptive,abbeel2007application,levine2016end,sheckells2017robust,NIPS2018_7638}.%\ac{Fluffify - who said it is important. what is the connection to ML? what can be gained? which algorithms are relevant? citations!!!}

Perhaps the most well-studied model in optimal control is Linear-Quadratic (LQ) control. In this model, both the state and the action are real-valued vectors. 
The dynamics of the environment are linear in both the state and action, and are perturbed by Gaussian noise; 
the cost is quadratic in the state and action vectors. 
When the costs and dynamics are known, the optimal control policy, which minimizes the steady-state cost, selects its actions as a linear function of the state vector, and can be derived by solving the algebraic Ricatti equations~\citep[e.g.,][]{bertsekas2005dynamic}.

Among the most challenging problems in LQ control is that of adaptive control: regulating a system with parameters which are initially unknown and have to be learned while incurring the associated costs.
This problem is exceptionally challenging since the system might become unstable. 
Specifically, the controller must control the magnitude of the state vectors; otherwise, its cost might grow arbitrarily large. 

\citet{abbasi2011regret} were the first to address the adaptive control problem from a learning-theoretic perspective. 
In their setting, there is a learning agent who knows the quadratic costs, yet has no knowledge regarding the dynamics of the system. 
The agent acts for $T$ rounds; at each round she observes the current state then chooses an action. %
Her goal is to minimize her regret, defined as the difference between her total cost and $T$ times the steady-state cost of the optimal policy---one that is computed using complete knowledge of the dynamics.

\citet{abbasi2011regret} gave $O(\sqrt{T})$-type regret bounds for LQ control where the dependency on the dimensionality is exponential, which was later improved by \citet{ibrahimi2012efficient} to a polynomial dependence.
% both show $\tO(\sqrt{T})$-type regret bounds for LQ control, where the dependency on the dimensionality is exponential in \citet{abbasi2011regret} and is improved to polynomial in \citet{ibrahimi2012efficient}.
%Both employ the principle of {\em optimism in face of uncertainty}, where they track a set of dynamics which are likely to include the true dynamics, and select the matrices that minimizes the cost. After selecting the matrices, and the corresponding optimistic action, this might result in disqualifying those matrices and shrinking the set of plausible matrices,  which is the progress that the online algorithm makes. While the approach is statistically optimal (in using the samples) there is a computational brier in computing the minimization over the set of plausible matrices. 
% The focus of both works was on the statistical aspects of the regret bound, and not on its computational aspects. 
However, the algorithms given in these works are not computationally efficient and require solving a complex non-convex optimization problem at each step.
%
%For this setting, \citet{abbasi2011regret} present an algorithm with $O(\sqrt{T})$ regret. The algorithm is initiated with a set of possible parameters that induce a set of possible policies---each optimal for a specific choice of parameters.  Moreover, the learner is guaranteed that each policy uniformly stabilizes the system with corresponding parameters. 
%
%Their algorithm maintains a confidence ellipsoid around a least squares estimate of the optimal parameters. Using this ellipsoid, the algorithm computes a policy by minimizing the steady-state cost over all feasible parameters. 
%
%However, despite being the first to address this problem, the algorithm is not applicable in most practical situations. 
%
%Indeed, solving the optimization problem used to compute a policy is believed to be computationally hard. Moreover, it is not clear how to construct the initial feasible set and ensure the uniform stability property.
%
%This results in a computationally-inefficient algorithm with a regret bound that has exponential dependence on problem constants. 
%
Developing an {\em efficient} algorithm with $O(\sqrt{T})$ regret has been a long standing open problem. 
Recently, \citet{dean2018regret} proposed a computationally-efficient algorithm attaining an $O(T^{2/3})$ regret bound,
%Their algorithm uses an explore-then-exploit scheme, and assumes that the learner is given an initial stabilizing policy. 
%This is a reasonable assumption, akin to be given an option to reset the system. 
%They further show how to utilize such a policy in an explore-then-exploit scheme, which provides a computationally-efficient algorithm attaining an $O(T^{2/3})$ regret bound.
%To solve it, one has to address the computational issues of \citet{abbasi2011regret}. First, it is necessary to devise an efficient way to compute control policies. Second, as there is no hope in stabilizing the system without any prior knowledge whatsoever, the algorithm must be initialized in a way that facilitates computational efficiency. 
%
%To alleviate both problems, \citet{dean2018regret} assume that the learner is given an initial stabilizing policy. This is a reasonable assumption, akin to be given an option to reset the system. They further show how to utilize such a policy in an explore-then-exploit scheme, which provides a computationally-efficient algorithm attaining an $O(T^{2/3})$ regret bound.
%
%Albeit being an efficient algorithm, its regret bound is not yet on par with the best-known $O(\sqrt{T})$ bound. \citet{dean2018regret} 
and stated as an open problem providing an $O(\sqrt{T})$ regret efficient algorithm.

In this paper, we give the first computationally-efficient algorithm that attains $\tO(\sqrt{T})$ regret for learning LQ systems, thus resolving the open problem of \citet{abbasi2011regret} and \citet{dean2018regret}.
%
% Our algorithm acts in two phases. 
% The first phase runs for a short period of $O(\sqrt{T})$ rounds during which the algorithm utilizes an initial stabilizing policy to roughly estimate the dynamics (a similar assumption is taken in \citealp{dean2018regret}).
% This rough estimate is then used as initialization for the second and main phase of the algorithm, ensuring the continued stability of the system.
% During the second phase, the 
The key to the efficiency of our algorithm is in reformulating the LQ control problem as a {\em convex} semi-definite program. 
Our algorithm solves a sequence of semi-definite relaxations of the infinite horizon LQ problem, the solutions of which are used to compute ``optimistic'' policies for the underlying unknown LQ system. 
As time progresses and the algorithm receives more samples from the system, these relaxations become tighter and serve as a better approximation of the actual LQ system.
% and computes ``optimistic'' policies ; then executes these policies. 
In this context, an optimistic policy is one that balances between exploration and exploitation; that is, between myopically utilizing its current information about the system parameters versus collecting new samples in order to obtain better estimates for subsequent predictions. 
%
% Our work draws inspiration from RL and applies it to a control-theoretic problem.
% This connection between the fields is a long lasting one, and has been fruitful in numerous applications in the past (\citealp{cohen2018online} for example). In future work, we aim to study it further in hope of it continuing to be beneficial to both fields of RL and control.

\subsection{Related work}

The techniques used in \citet{abbasi2011regret,ibrahimi2012efficient} as well as those in this paper, draw inspiration from the UCRL algorithm \citep{jaksch2010near} for learning in unknown Markov Decision Processes (MDPs). 
The main methodology is that of ``optimism in the face of uncertainty'' that has been highly influential in the reinforcement learning literature~\citep{lai1985asymptotically,brafman2002r}.

Over the years, techniques from reinforcement learning have been applied extensively in control theory. 
In particular, many recent works were published on the topic of learning LQ systems; these are \citet{abbasi2011regret,ibrahimi2012efficient,faradonbeh2017finite,abbasi2018regret,arora2018towards,fazel2018global,malik2018derivative} to name a few. 
%Nonetheless, none of these that approach the adaptive control problem do not present computationally-efficient algorithms.
%For examples, see \cite{bradtke1994adaptive,abbeel2007application,levine2016end,sheckells2017robust,NIPS2018_7638}.

It is also worth noting an orthogonal line of works that attempts to adaptively control LQ systems using Thompson sampling, most notably \citet{abeille2017thompson,ouyang2017learning,abeille2018improved}. 
%
%Similarly to frequentest approaches, 
Unfortunately, these works are also concerned with the statistical aspects of the problem, and none of them present computationally-efficient algorithms. 
%Indeed, the algorithms therein require sampling a sufficiently stable policy, yet do not guarantee a bound on on the time needed to obtain such a sample.

\section{Preliminaries}

\paragraph{Notation.}

The following notation will be used throughout the paper.
We use $\spnorm{\cdot}$ to denote the operator norm, that is, $\spnorm{M} = \max_{x:\norm{x}=1} \norm{Mx}$ is the maximum singular value of a matrix $M$, and $\trnorm{\cdot}$ to denote the trace norm, $\trnorm{M} = \trace(\sqrt{M\tr M})$.
The notation $\rho(M)$ refers to the spectral radius of a matrix $M$, i.e., $\rho(M)$ is the largest absolute value of its eigenvalues.%
\footnote{Note that for a non-symmetric matrix $M$ (as would often be the case in the sequel), the spectral radius can be very different from the operator norm of $M$. In particular, it could be the case that $\rho(M)<1$ yet $\spnorm{M} \gg 1$.}
Finally, we use the $A \bullet B$ to denote the entry-wise dot product between matrices, namely $A \bullet B = \trace(A\tr B)$.

% \subsection{Linear-Quadratic Control}
\subsection{Problem Setting and Background}

\paragraph{Linear-Quadratic Control.}

We consider the problem of adaptively controlling an unknown discrete-time {\em Linear-Quadratic Regulator} (LQR) over $T$ rounds.
% A learner sequentially interacts with an LQR system for $T$ rounds. 
At time $t$, a learner observes the current state of the system, which is a vector $x_t \in \bbR^d$, and chooses an action $u_t \in \bbR^k$.
Thereafter, the learner incurs a cost $c_t$, and the system transitions to the next state $x_{t+1}$, both of which are defined as follows:
\begin{equation} \label{eq:lqr}  
\begin{alignedat}{2}
    &c_t &&= x_t\tr Q x_t + u_t\tr R u_t
    ~; \\
    &x_{t+1} &&= \Ast x_t + \Bst u_t + w_t
    ~.
\end{alignedat}
\end{equation}
Here, $Q \in \reals^{d \times d}$ and $R \in \reals^{k \times k}$ are positive-definite matrices, $w_t \sim \calN(0,W)$ is an i.i.d.~zero-mean Gaussian vector with covariance $W$, and $\Ast \in \reals^{d \times d}$ and $\Bst \in \reals^{d\times k}$ are real valued matrices.
We henceforth denote $n = d+k$, so that the augmented matrix $(\Ast \, \Bst)$ is of dimension $d \times n$.

A (stationary and deterministic) policy $\pi : \bbR^d \mapsto \bbR^k$ maps the current state $x_t$ to an action $u_t$. 
The cost of the policy after $T$ time steps is
\[
    J_T(\pi) 
    % = 
    % \sum_{t=1}^T c_t
    =
    \sum_{t=1}^T \Lr{ x_t\tr Q x_t + u_t\tr R u_t }
    ~,
\]
where $u_1,\ldots,u_{T}$ are chosen according to $\pi$ starting from some fixed state $x_1$. %The expectation is taken with respect to the randomness in the state transitions. 
In the infinite-horizon version of the problem, the goal is to minimize the steady-state cost $J(\pi) = \lim_{T \rightarrow \infty}  \tfrac{1}{T} \bbE[J_T(\pi)]$.

As is standard in the literature, we assume that the system~\eqref{eq:lqr} is controllable,%
\footnote{The system~\eqref{eq:lqr} is said to be controllable when the matrix $\Lr{\Bst \; \Ast \Bst  \; \cdots \; \Ast^{d-1} \Bst}$ is of full rank.}
in which case the optimal policy that minimizes $J(\pi)$ is \emph{linear}, i.e., has the form $\pi^\st(x) = \Kst x$ for some matrix $\Kst \in \bbR^{k \times d}$. 
For the optimal policy $\pi^\st$ we denote $J(\pi^\st) = J^\st$.

A policy $\pi(x)=Kx$ is \emph{stable} if the matrix $\Ast+ \Bst K$ is stable, that is, if $\rho(\Ast+\Bst K) < 1$.
For a stable policy $\pi$ we can define a cost-to-go function $x_1 \mapsto x_1\tr P x_1$ that maps a state $x_1$ to the total additional expected cost of $\pi$ when starting from $x_1$. Concretely, we have $x_1\tr P x_1 = \sum_{t=1}^\infty \Lr{\bbE[c_t] - J(\pi)}$. 
% For a stable policy of the form $\pi(x) = K x$ a classic result states that (if $Q \succ 0$) $P$ is a positive definite matrix that satisfies:
% \[
%     P 
%     = 
%     Q + K\tr R K + (\Ast+\Bst K)\tr P (\Ast + \Bst K)
%     ~,
% \]
% and the steady-state cost is $J(\pi) = P \bullet W$.
% \tk{write instead the inequalities for $P^\st$ and number/crossref them}
For the optimal policy $\pi^\st(x) = \Kst x$, a classic result \citep{whittle1996optimal,bertsekas2005dynamic} states that the matrix $P^\st$ associated with its cost-to-go function is a positive definite matrix that satisfies:
\begin{equation} \label{eq:bellman}
    P^\st 
    \preceq
    Q + K\tr R K + (\Ast+\Bst K)\tr P^\st (\Ast + \Bst K)
\end{equation}
for any matrix $K \in \bbR^{k \times d}$, with equality when $K = \Kst$:
\begin{equation} \label{eq:ricatti}
    P^\st 
    =
    Q + \Kst\tr R \Kst + (\Ast+\Bst \Kst)\tr P^\st (\Ast + \Bst \Kst)~.
\end{equation}
Furthermore, the optimal steady-state cost $J^\st$ equals $P^\st \bullet W$.

% The optimal policy $\pi^\st$ satisfies the Bellman equations with respect to $P^\st$:
% \begin{equation}
%     \label{eq:bellman}
%     x\tr P^\st x 
%     = 
%     \min_{u \in \bbR^k} x\tr Q x + u\tr R u + (\Ast x + \Bst u)\tr P^\st (\Ast x + \Bst u)~,\quad \forall x~.
% \end{equation}
% which in matrix form these are known as the Ricatti equations:
% \[
%     P^\st = Q + \Ast\tr P^\st \Ast - \Ast\tr P^\st \Bst (R + \Bst\tr P^\st \Bst)^{-1} \Bst\tr P^\st \Ast~.
% \]
% Moreover, fixing $x$, the minimal $u$ in the RHS of \cref{eq:bellman} is attained at $u = \Kst x$ for $\Kst = -\Lr{R+ \Bst\tr P^\st \Bst}^{-1} \Bst\tr P^\st \Ast$.

\paragraph{Problem definition.}

We henceforth consider a learning setting in which the learner is uninformed about the dynamics of the system. 
Namely, the matrices $\Ast$ and $\Bst$ in \cref{eq:lqr} are fixed but unknown to the learner. 
For simplicity, we assume that the cost matrices $Q$ and $R$ are fixed and known; a straightforward yet technical adaptation of our approach can handle uncertainties in these matrices as well.

A learning algorithm is a mapping from the current state $x_t$ and previous observations $\set{x_s,u_s}_{s=1}^{t-1}$ to an action $u_t$ at time $t$.
% The cost of the algorithm is then $\sum_{t=1}^T x_t\tr Q x_t + u_t\tr R u_t$ where $u_1,\ldots,u_{T}$ are chosen according to $\calA$.
%
An algorithm is measured by its $T$-round regret, defined as the difference between its total cost over $T$ rounds and $T$ times the steady-state cost of the optimal policy which knows both $\Ast$ and $\Bst$. 
That is,
\[
    R_T 
    % = 
    % J_T(\calA) - T \cdot J^\st
    =
    \sum_{t=1}^T \Lr{ x_t\tr Q x_t + u_t\tr R u_t - J^\st }
    ~,
\]
where $u_1,\ldots,u_{T}$ are the actions chosen by the algorithm and $x_1,\ldots,x_T$ are the resulting states.

\paragraph{Our assumptions.}

We make the following assumptions about the LQ system \eqref{eq:lqr}:
\begin{enumerate}[label=(\roman*),nosep]
% \item
% The system \eqref{eq:lqr} is controllable;%
% \footnote{Controllability was in fact already assumed implicitly in the definition of regret above.}
\item \label{ass:1}
there are known positive constants $\alpha_0, \alpha_1, \sigma, \vartheta, \nu > 0$ such that
\begin{align*}
    \alpha_0 I \preceq Q \preceq \alpha_1 I , \quad
    \alpha_0 I \preceq R \preceq \alpha_1 I , \quad
    %\\
    W = \sigma^2 I ,\quad %\footnotemark \quad
    \spnorm{\lr{\Ast \, \Bst}\!} \leq \vartheta , \quad
    J^\st \leq \nu ;
\end{align*}
\item \label{ass:2}
there is a policy $K_\ws \in \reals^{k \times d}$, known to the learner, which is stable for the LQR \eqref{eq:lqr}.
\end{enumerate}
%
% \footnotetext{This assumption is only made for simplicity, and in fact, our analysis only requires upper and lower bounds on the eigenvalues of $W$.}

Assumption~\ref{ass:1} is rather mild and only requires having upper and lower bounds on the unknown system parameters.
We remark that the assumption $W = \sigma^2 I$ is made only for simplicity, and in fact, our analysis only requires upper and lower bounds on the eigenvalues of $W$.
Assumption~\ref{ass:2}, which has already appeared in the context of learning in LQRs~\citep{dean2018regret}, is also not very restrictive.
In realistic systems, it is reasonable that one knows how to ``reset'' the dynamics and prevent them from reaching unbounded states.
Further, in many cases a stabilizing policy can be found efficiently~\citep{dean2017sample}.

% One of the main challenges in learning an LQ controller is stabilizing it. 
% Indeed, the learner cannot hope to stabilize the system without sufficient prior knowledge; otherwise, during learning, the learner is powerless to control the norms of the state vectors which can become exponential in problem parameters. 
% This would immediately result in exponentially large regret.

% For this reason, most previous work assumes that the optimal policy belongs to a known set of ``sufficiently-stable'' policies, and the resulting algorithms play according to policies from this set. 
% Nevertheless, even with complete knowledge of the dynamics, finding the optimal policy from this set is computationally-hard. 

% To alleviate this issue, \citet{dean2018regret} assume a given initial stable policy. 
% This assumption is reasonable for many systems, akin to be given the option to reset the system. 
% The authors proceed in using this policy to perform an explore-then-exploit scheme that results in an $O(T^{2/3})$ regret bound. 
% %
% In contrast, in our work we only use the stable policy to perform initial exploration that guarantees generating stabilizing policies in the next phase of our algorithm.

\subsection{SDP Formulation of LQR}

A key step in our approach towards the design of an \emph{efficient} learning algorithm is in reformulating the planning problem in LQRs as a convex optimization problem.
To this end, we make use of a semidefinite formulation introduced in \citet{cohen2018online} that would allow us to find the optimal cost of the LQ system \eqref{eq:lqr}:
\begin{alignat}{2}
    &\text{minimize} \quad && 
    \qrmatrix \bullet \Sigma 
    \nonumber \\
    &\text{subject to} \quad && 
    \Sigma_{xx} = \Lr{\Ast \; \Bst} \Sigma \Lr{\Ast \; \Bst}\tr + W~,
    \label{eq:lqr-sdp}
    \\
    & && \Sigma \succeq 0 
    %,\; \trace(\Sigma) \le \nu 
    \nonumber
    ~.
\end{alignat}
Here, $\Sigma$ is an $n \times n$ PSD matrix, with $n = d+k$, that has the following block structure:
% Denote $n = d+k$.
% Formally, $\Sigma \in \bbR^{n \times n}$ is split into four sub-matrices:
$$
    \Sigma 
    = 
    \begin{pmatrix} 
    \Sigma_{xx} & \Sigma_{xu} \\
    \Sigma_{ux} & \Sigma_{uu}
    \end{pmatrix}
    ,
$$
where $\Sigma_{xx} \in \bbR^{d \times d}$, $\Sigma_{ux} = \Sigma_{xu}\tr \in \bbR^{k \times d}$ and $\Sigma_{uu} \in \bbR^{k \times k}$.
The matrix $\Sigma$ represents the covariance matrix of the joint distribution of $(x,u)$ when the system is in its steady-state. 

As was established in \citet{cohen2018online}, the optimal value of the program is exactly the infinite-horizon optimal cost $J^\st$. 
Moreover, when $W \succ 0$, the optimal policy of the system $\Kst$ can be extracted from an optimal $\Sigma$ via $K = \calK(\Sigma)$ where $\calK(\Sigma) = \Sigma_{ux} \Sigma_{xx}^{-1}$. 
In fact, when the LQ system follows \emph{any} stable policy $K$, the state vectors converge to a steady-state distribution whose covariance matrix is denoted by $X = \E[x x\tr]$, and the matrix
$
    \calE(K) 
    =
    \Lr{\!\begin{smallmatrix}
    X & X \smash{K\tr} \\
    K X & K X \smash{K\tr}
    \end{smallmatrix}\!}
$
is feasible for the SDP. This particularly implies that the optimal solution $\Sigma^\st$ is of rank $d$ and has the form $\Sigma^\st = \calE(\Kst)$.
This is formalized as follows.

\begin{theorem*}[\citealp{cohen2018online}]
Let $\Sigma$ be any feasible solution to the SDP~\eqref{eq:lqr-sdp}, and let $K = \calK(\Sigma)$. 
Then the policy $\pi(x) = K x$ is stable for the LQR \eqref{eq:lqr}, and it holds that $\calE(K) \preceq \Sigma$. 
In particular, $\calE(K)$ is also feasible for the SDP and its cost is at most that of $\Sigma$.
\end{theorem*}

\subsection{Strong Stability}

The quadratic cost function is unbounded. Indeed, it might be that the norms of the state vectors $x_1,x_2,\ldots$ grow exponentially fast resulting in poor regret for the learner.

To alleviate this issue we rely on the notion of a strongly-stable policy, introduced by \citet{cohen2018online}.
Intuitively, strongly-stable policies are ones in which the norms of the state vectors remain controlled.

\begin{definition}[strong stability]
A matrix $M$ is \emph{$(\kappa,\gamma)$-strongly stable} (for $\kappa \ge 1$ and $0 < \gamma \le 1$) if there exists matrices $H \succ 0$ and $L$ such that $M = HLH^{-1}$, with $\spnorm{L} \le 1-\gamma$ and $\spnorm{H}\spnorm{H^{-1}} \le \kappa$.
% \end{definition}

% \begin{definition}[strongly stable policy]
A policy $K$ for the linear system \eqref{eq:lqr} is \emph{$(\kappa,\gamma)$-strongly stable} (for $\kappa \ge 1$ and $0 < \gamma \le 1$) if $\spnorm{K} \le \kappa$ and the matrix $\Ast+\Bst K$ is $(\kappa,\gamma)$-strongly stable.
\end{definition}

We note that, in particular, any stable policy $K$ is in fact $(\kappa,\gamma)$-strongly stable for some $\kappa,\gamma > 0$ (see \citealp{cohen2018online} for a proof).
Our analysis requires a stronger notion that pertains to the stability of a sequence of policies, also borrowed from \citet{cohen2018online}.

\begin{definition}[sequential strong stability] \label{def:seq-strong-stab}
A sequence of policies $K_1,K_2,\ldots$ for the linear dynamics in \cref{eq:lqr} is $(\kappa, \gamma)$-strongly stable (for $\kappa > 0$ and $0 < \gamma \le 1$) if there exist matrices $H_1,H_2,\ldots \succ 0$ and $L_1,L_2,\ldots$ such that $\Ast + \Bst K_t = H_t L_t H_t^{-1}$ for all $t$, with the following properties:
\begin{enumerate}[label=(\roman*)]
    \item $\spnorm{L_t} \le 1- \gamma$ and $\spnorm{K_t} \le \kappa$;
    \item $\spnorm{H_t} \le B_0$ and $\spnorm{H_t^{-1}} \le 1/b_0$ with $\kappa = B_0 / b_0$;
    \item $\spnorm{H_{t+1}^{-1} H_t} \le 1+\gamma / 2$.
\end{enumerate}
\end{definition}

For a sequentially strongly stable sequence of policies one can show that the expected magnitude of the state vectors remains controlled; for completeness, we include a proof in~\cref{sec:strg-stab-norm-proof}.

\begin{lemma} \label{lem:strg-stab-norm}
Let $x_1,x_2,\ldots$ be a sequence of states starting from a deterministic state $x_1$, and generated by the dynamics in \cref{eq:lqr} following a $(\kappa,\gamma)$-strongly stable sequence of policies $K_1,K_2,\ldots$.
Then, for all $t \geq 1$ we have
\begin{align*}
    \norm{x_t}
    \leq
    \kappa e^{-\gamma (t-1)/2} \norm{x_1} + \frac{2\kappa}{\gamma} \max_{1 \leq s < t} \norm{w_t}
    .
\end{align*}
\end{lemma}

% \subsection{Main Results}

% TK: we can actually state the results as early as here.

% \begin{theorem} \label{thm:main}
% There exists an algorithm (\cref{alg:main} initialized with the warm-up procedure of \cref{alg:warmstart}) for which, with probability at least $1-\delta$, we have
% \[
%     R_T 
%     =
%     \O\LR{\sqrt{T \log^5(T/\delta)}}
%     ~,
% \]
% where the big-$O$ notation hides factors polynomial in $\nu, \sigma^{-1}, \kappa_{\ws}, \gamma_{\ws}^{-1}, L, \beta, \alpha_0^{-1}, d, k$. 
% Moreover, the running time per round of the algorithm is polynomial in these factors and in $\log{T}$ and $\log(1/\delta)$.
% \end{theorem}

\begin{algorithm}[ht]
\caption{OSLO: Optimistic Semi-definite programming for Lq cOntrol} \label{alg:main}
\begin{algorithmic}[1] 
    \STATE {\bf input}:
        parameters $\alpha_0, \sigma^2, \vartheta, \nu > 0$;
        confidence $\delta \in (0,1)$; 
        and an initial estimate $(A_0 \, B_0)$ such that $\norm{(A_0 \, B_0) - (\Ast \, \Bst)}_\frob^2 \leq \epsilon$.
    \STATE {\bf initialize}: $\mu = 5\vartheta\sqrt{T}$, $V_1 = \lambda I$ where
    %\vspace{-0.2cm}
    \begin{align*}
        \lambda 
        = 
        \frac{2^{11}\nu^{5}\vartheta\sqrt{T}}{\alpha_0^{5} \sigma^{10}}
        \;\; \text{and} \; \;
        \beta
        =
        \frac{2^{18} \nu^4 n^2}{\alpha_0^4 \sigma^6} \log\frac{T}{\delta}
        ~.
    \end{align*}
    %\vspace{-0.3cm}
    %
    \FOR{$t=1,\ldots,T$}
        \STATE {\bf receive} state $x_t$.
        \IF{$\det(V_t) > 2 \det(V_{\tau})$ or $t = 1$}
            \STATE {\bf start new episode}: $\tau = t$.
            \STATE {\bf estimate system parameters}: \label{ln:estimation}
            Let $(A_t \, B_t)$ be a minimizer of
            %\vspace{-0.2cm}
            {%\small
            \begin{align*}
%                \min_{\lr{A \, B} \in \reals^{d \times n}} \;
%            	\lrset{ 
            	\frac{1}{\beta} \sum_{s=1}^{t-1}\norm{ \lr{A \, B} z_{s} - x_{s+1} }^{2} 
            	+ 
            	\lambda \norm{\lr{A \, B} - \lr{A_0 \, B_0}}_\frob^2
%            	}
            \end{align*}
            }%
            over all matrices $\lr{A \, B} \in \reals^{d \times n}$.
            \STATE {\bf compute policy}: \label{ln:sdp}
            let $\Sigma_t \in \reals^{n \times n}$ be an optimal solution to the SDP program:
            %\vspace{-0.1cm}
            {%\small
            \begin{alignat*}{2}
                &\text{min} \;
                &&\Sigma \bullet \qrmatrix
                \nonumber \\
                &\text{s.t.} 
                &&\Sigma_{xx} 
                \succeq 
                (A_t \; B_t) \Sigma (A_t \; B_t)\tr + W - \mu (\Sigma \bullet V_t^{-1}) I~, \;
                \\
                &&& \Sigma \succeq 0
                ~.
            \end{alignat*}}
            %\vspace{-0.5cm}
            %
            \STATE \label{ln:K}
                {\bf set} $K_{t} = (\Sigma_t)_{ux} \, (\Sigma_t)_{xx}^{-1}$.
        \ELSE
            \STATE {\bf set} $K_{t} = K_{t-1}$, $A_t = A_{t-1}$, $B_t = B_{t-1}$.
        \ENDIF
        \STATE {\bf play} $u_t = K_t x_t$.
        \STATE {\bf update} $z_t = \begin{psmallmatrix} x_t \\ u_t \end{psmallmatrix}$ and $V_{t+1} = V_{t}+ \beta^{-1} z_t z_t\tr$.
    \ENDFOR
\end{algorithmic}
\end{algorithm}

\section{Efficient Algorithm for Learning in LQRs}

In this section we describe our efficient online algorithm for learning in LQRs;
see pseudo-code in \cref{alg:main}.
The algorithm receives as input the parameters $\alpha_0$, $\nu$, $\sigma^2$ and $\vartheta$, 
% that are used to determine the variables $\lambda$ and $\mu$ of the relaxed SDP (line~\ref{ln:sdp}).
%
further requires an initial estimate $\lr{A_0 \, B_0}$ that approximates the true parameters $\lr{\Ast \, \Bst}$ within an error $\epsilon$.
As we later show, this estimate only needs to be accurate to within $\epsilon = O(1/\sqrt{T})$ of the true parameters, and 
we can make sure this is satisfied by employing a known stabilizing policy $K_\ws$ for exploration over $O(\sqrt{T})$ rounds.
% 
% In \cref{sec:warmstart}, we show how to set up these initial conditions using a known stabilizing policy $K_\ws$ for the system \eqref{eq:lqr}.

We next describe in detail the main steps of the algorithm.
The algorithm maintains estimates $\lr{A_t \, B_t}$ of the true parameters $\lr{\Ast \, \Bst}$ that improve from round to round, as well as a PD matrix $V_t \succ 0$ that represents a confidence ellipsoid around the current estimates $\lr{A_t \, B_t}$. 
The algorithm proceeds in epochs, each starting whenever the volume of the ellipsoid is halved and consists of the following steps.

% In light of that, in the sequel we show that there are only $\wt O(n)$ epochs with high probability. 
% This is important because the policy $K_t$ is chosen as to minimize its steady-state cost. 
% Therefore, the learner might suffer some short-term additive penalty each time the algorithm switches its internal policy. 
% % (In principle, we could tolerate as much as $O(\sqrt{T})$ policy changes, but there is no cost in limiting this number to $\O(\log{T})$ with high probability.)

% Each epoch of the algorithm consists of the following key steps.

\subsection{Estimating parameters}

The first step of each epoch is standard: we employ a least-squares estimator (in \cref{ln:estimation}) to construct a new approximation $(A_t \, B_t)$ of the parameters $(\Ast \, \Bst)$ based on the observations $z_t$ collected so far.
The confidence bounds of this estimator are given in terms of the covariance matrix $V_t$ of the vectors $z_1,\ldots,z_{t-1}$.

\subsection{Computing a policy via an SDP}

The main step of the algorithm takes place in line~\ref{ln:sdp} of \cref{alg:main}, where we form a ``relaxed'' SDP program based on the current estimates $(A_t \, B_t)$ and the corresponding confidence matrix $V_t$, and solve it in order to compute a stable policy for the underlying LQR system.
The idea here is to adapt the SDP formulation \eqref{eq:lqr-sdp} of the LQR system, whose description needs the true underlying parameters, to an SDP program that only relies on estimates of the true parameters and accounts for the uncertainty associated with them.
Once the relaxed SDP is solved, extracting a (deterministic) policy $K_t$ from the solution $\Sigma_t$ is done in the same way as in the case of the exact SDP~\eqref{eq:lqr-sdp}.

The relaxed SDP incorporates a relaxed form of the inequality constraint in \eqref{eq:lqr-sdp}; as we show in the analysis, this program is a relaxation of the ``exact'' SDP \eqref{eq:lqr-sdp} provided that the estimates $(A_t \, B_t)$ are sufficiently accurate
(this is one place where having fairly accurate initial estimates as input to the algorithm is useful).
In other words, the relaxed SDP always underestimates the steady-state cost of the optimal policy of the LQR \eqref{eq:lqr}.
In this sense, \cref{alg:main} is ``optimistic in the face of uncertainty''~\citep[e.g.,][]{brafman2002r,jaksch2010near}.
% This is exactly the so-called ''optimism in the face of uncertainty'' principle that is heavily-used for learning MDPs~\citep{brafman2002r,jaksch2010near}.

\subsection{Exploring, exploiting, and updating confidence}

After retrieving a policy $K_t$, the algorithm takes action: it computes $u_t = K_t x_t$, which is the action recommended by policy~$K_t$ at state $x_t$, and then plays $u_t$ and updates the confidence matrix $V_t$ with the new observations at step $t$.
The policy $K_t$ therefore serves and balances two goals---\emph{exploitation} and \emph{exploration}---as it is used both as a ``best guess'' to the optimal policy (based on past observations), as well as means to collect new samples and obtain better estimates of the system parameters in subsequent steps of the algorithm.

% \paragraph{Exploration vs.\ exploitation.}

% \ac{What do you think?}
% The solution $\Sigma_t$ of the relaxed SDP represents the joint covariance matrix of the pair $(x,u)$ at the steady-state distribution of the policy $K_t$. 
% The term $\mu (\Sigma_t \bullet V_t^{-1}) I$ in the SDP constraint forms a relaxation which encourages $K_t$ to explore rather than exploit using its current estimates $\lr{A_t \, B_t}$.

% Note, however, that the policy $K_t$ is computed to optimize the steady-state cost of the system, and consequently there is no short-term gain when switching to it from the previously used policy. \cref{alg:main} accordingly switches to a new policy only when the volume of the confidence ellipsoid is halved, consequently limiting the total number of policy switches to just $\wt O(n)$ with high probability.

\section{Overview of Analysis}
\label{sec:analysis}

We now formally state our main result: a high-probability $\tO(\sqrt{T})$ regret bound for the efficient algorithm given in \cref{alg:main}.

\begin{theorem} \label{thm:main}
Suppose that \cref{alg:main} is initialized so that the initial estimation error 
$
    \norm{(A_0 \, B_0) - (\Ast \, \Bst)}_\frob^2
    \le 
    \epsilon
$ 
satisfies
\begin{align*}
    \epsilon
    \leq
    \frac{1}{4\lambda}
    =
    \frac{\alpha_0^{5} \sigma^{10}}{2^{13}\nu^{5}\vartheta\sqrt{T}}
    ~.
\end{align*}
% Assume $T \ge \tfrac{2^{30} d \nu^{10}\vartheta^2}{\alpha_0^{10} \sigma^{20}} + \norm{x_1}^2 + \vartheta^{-2} (1 + \norm{x_1}^2)^2 + \sigma^2$.
Assume $T \geq \poly(n,\nu,\vartheta,\alpha_0^{-1},\sigma^{-1},\norm{x_1})$.
Then for any $\delta \in (0,1)$, with probability at least $1-\delta$ the regret of \cref{alg:main} satisfies
\begin{align*}
    R_T
    = 
    O \biggl(
    \frac{\nu^5 n^3 \vartheta}{\alpha_0^4 \sigma^8} \sqrt{T \log^4 \frac{T}{\delta}}
    +
    \nu \sqrt{T \log^3 \frac{T}{\delta}}
    \biggr)
    ~. 
\end{align*}
Furthermore, the run-time per round of the procedure is polynomial in these factors.
\end{theorem}

% The proof of \cref{thm:main} is laid out in \cref{sec:analysis} below.

\begin{remark*}
At first glance it may appear that the regret bound of \cref{thm:main} becomes worse as the noise variance $\sigma^2$ becomes smaller. 
This seems highly counter-intuitive and, indeed, is not true in general. 
This is because when $\sigma$ is small we also expect the bound on the optimal loss $\nu$ to be small. 
In particular, suppose that $\Kst$ is $(\kappa_\st,\gamma_\st)$-strongly stable; then, one can show that $J^\st \le \sigma^2 \alpha_1 \kappa_\st^2 / \gamma_\st$. 
Plugging this as $\nu$ into the bound of \cref{thm:main} 
% gives
% \[
%     R_T
%     =
%     O 
%     \biggl(
%     \sigma^2 \alpha_1 \biggl( \frac{\kappa_\st^2}{\gamma_\st} \biggr)^5 \biggl( \frac{\alpha_1}{\alpha_0} \biggr)^4
%     n^3 \vartheta+
%     \sqrt{T \log^4 \frac{T}{\delta}}
%     +
%     \nu \sqrt{T \log^3 \frac{T}{\delta}} \biggr)
%     % =
%     % \tO_T(\sqrt{T})
%     ~.
% \]
% \[
%     R_T
%     =
%     O 
%     \biggl(
%     \sigma^2 \alpha_1 \biggl( \frac{\kappa_\st^2}{\gamma_\st} \biggr)^5 \biggl( \frac{\alpha_1}{\alpha_0} \biggr)^4
%     n^3 \vartheta
%     \sqrt{T} \log^2 \frac{T}{\delta}
%     +
%     \frac{\sigma^2 \alpha_1 \kappa_\st^2}{\gamma_\st} \sqrt{T \log^3 \frac{T}{\delta}}
%     \biggr)~.
% \]
reveals a linear dependence in $\sigma^2$.
\end{remark*}

In \cref{sec:warmstart} we show how to set up the initial conditions of \cref{thm:main};
we utilize a stable (but otherwise arbitrary) policy given as input and show the following.

\begin{corollary} \label{corr:regretplusws}
Suppose we are provided a policy $K_{\ws}$ which is known to be $(\kappa_{\ws}, \gamma_{\ws})$-strongly stable for the LQR \eqref{eq:lqr}.
Assume $T \geq \poly(n,\nu,\vartheta,\alpha_0^{-1},\sigma^{-1},\kappa_0,\gamma_0^{-1},\log(\delta^{-1}))$.
% \begin{align*}
%     T 
%     \ge 
%     \frac{2^{30} d \nu^{10}\vartheta^2}{\alpha_0^{10} \sigma^{20}}
%     + \frac{300 \sigma^2 \kappa_{\ws}^4 (n + k \vartheta^2 \kappa_\ws^2)}{\gamma_{\ws}^2 \min\{\vartheta^2,1\}} \log\LR{\frac{300 \sigma^2 \kappa_{\ws}^4 (n + k \vartheta^2 \kappa_\ws^2)}{\gamma_{\ws}^2 \min\{\vartheta^2,1\} \delta}}
%     + \sigma^2
%     ~.
% \end{align*}
Suppose at first we utilize $K_{\ws}$ in the warm-up procedure of \cref{alg:warmstart} for 
% \[
%     T_{\ws} 
%     = 
%     \biggl\lceil 
%     \frac{2^{25} n^2 \nu^{5}\vartheta}{\alpha_0^{5} \sigma^{10}} 
%     \sqrt{T \log^2 \frac{T}{\delta}}
%     \biggr\rceil 
% \]
\[
    T_{\ws} 
    =
    \Theta\lr{
    \frac{n^2 \nu^{5}\vartheta}{\alpha_0^{5} \sigma^{10}} 
    \sqrt{T \log^2 \frac{T}{\delta}}
    }
\]
rounds;
thereafter, we run \cref{alg:main}.
Then, the initial conditions of \cref{thm:main} hold by the end of the warm-up phase, and with probability at least $1-\delta$ the regret of the overall procedure is bounded as
\begin{align*}
    R_T
    =
    \O \biggl( \frac{\alpha_1 n^2 \nu^{5} \vartheta \kappa_{\ws}^4}{\alpha_0^{5} \sigma^{8} \gamma_{\ws}^2} 
    (n + k \vartheta^2 \kappa_\ws^2) \sqrt{T \log^4 \frac{T}{\delta}}
    %\\ 
    %&
    + 
    \nu \sqrt{T \log^3 \frac{T}{\delta}}
    \biggr)
 %   =  \wt O(\sqrt{T})
    ~. 
\end{align*}
Furthermore, the runtime per round of the procedure is polynomial in these factors and in $T,\log(1/\delta)$.
\end{corollary}

In the remainder of the section, we give an overview of the main steps in the analysis, delegating the technical proofs to later sections and appendices.

\subsection{Parameters estimation}

\cref{alg:main} repeatedly computes least-square estimates of $\lr{\Ast \; \Bst}$.
The next theorem, similar to one shown in \cite{abbasi2011regret}, yields a high-probability bound on the error of this least-squares estimate.

\begin{lemma} \label{lem:concentration}
Let $\Delta_t = \lr{A_t \, B_t} - \lr{\Ast \, \Bst}$.
For any $\delta \in (0,1)$, with probability at least $1-\delta$,
\[
    \trace(\Delta_t V_t \Delta_t\tr) 
    \leq
    \frac{4 \sigma^2 d}{\beta} \log \LR{\frac{d}{\delta} \frac{\det(V_t)}{\det(V_1)}}
    + 
    2 \lambda \, \norm{\Delta_0}_\frob^2
    ~.
\]
In particular, when $\norm{\Delta_0}_\frob^2 \le 1/(4\lambda)$ 
% where $\lambda \ge 1$, $\beta, V_t, V_1$ are chosen according to \cref{alg:main} 
and $\sum_{s=1}^t \norm{z_s}^2 \le 2 \beta T$,
one has
$
    \trace(\Delta_t V_t \Delta_t\tr) \le 1
$.
\end{lemma}

We see that the boundness of the states $z_t$ (specifically, the fact that they do not grow exponentially with $t$) is crucial for the estimation.
Below, we will show how the policies computed by the algorithm ensure this condition.

The proof of \cref{lem:concentration} is based on a self-normalized martingale concentration inequality due to \citet{abbasi2011improved}; for completeness, we include a proof in \cref{sec:proofofconcentration}.

\subsection{Policy computation via a relaxed SDP}

Next, assume that the estimates $A_t,B_t$ of $\Ast,\Bst$ computed in the previous step are indeed such that the error $\Delta_t = (A_t \, B_t) - (\Ast \, \Bst)$ has $\trace(\Delta_t V_t \Delta_t\tr) \leq 1$ for the confidence matrix $V_t = \lambda I + \beta^{-1} \sum_{s=1}^{t-1} z_s z_s\tr$. 

Consider the relaxed SDP program solved by the algorithm in \cref{ln:sdp}.
The following lemma follows from the optimality conditions of the SDP and will be used to extract a stable policy from the SDP solution, and to relate the cost of actions taken by this policy to properties of the SDP solutions.
This lemma, together with \cref{lem:seqstrong-main} below, summarize the key consequences of the relaxed SDP formulation that central to our approach; we elaborate more on the relaxed SDP and its properties in \cref{sec:sdp-analysis} below.

\begin{lemma} \label{lem:relaxed-sdp-main}
Assume the conditions of \cref{thm:main}, and further that
% $
%     \mu 
%     \geq 
%     1 + 2 \vartheta \spnorm{V_t}^{1/2}
% $
% and
% $
%     V_t 
%     \succeq 
%     % (\nu \mu / \alpha_0 \sigma^2) I
%     \kappa^2 \mu I
%     .
% $
$
    \norm{V_t} \leq 4T 
$.
Then the SDP solved in \cref{ln:sdp} of the algorithm is a relaxation of the exact SDP \eqref{eq:lqr-sdp}, and we have:
\begin{enumerate}[label=(\roman*),nosep]
\item
the value of the optimal solution is at most $J^\st \le \nu$ which implies $\trnorm{\Sigma_t} \le J^\st / \alpha_0$;
\item
$(\Sigma_t)_{xx}$ is invertible and so the policy $K_t = (\Sigma_t)_{ux} (\Sigma_t)_{xx}^{-1}$ is well defined;
\item
there exists a positive semi-definite matrix $P_t \succeq 0$ with $\trnorm{P_t} \le J^\st / \sigma^2$ such that
\begin{align*}
    % P_t
    % &=
    % Q + K_t\tr P_t K_t + (A_t + B_t K_t)\tr P_t (A_t + B_t K_t) - \mu \trnorm{P_t} \IhatKt\tr V_t^{-1} \IhatKt
    % ~, \;\;\text{and} \\
    P_t
    \succeq
    Q + K_t\tr P_t K_t + (\Ast + \Bst K_t)\tr P_t (\Ast + \Bst K_t) %\\
    %&\quad 
    - 2 \mu \trnorm{P_t} \IhatKt\tr V_t^{-1} \IhatKt~.
\end{align*}
\end{enumerate}
\end{lemma}

The positive definite matrix $P_t$ in the above lemma is in fact the dual variable corresponding to the optimal solution $\Sigma_t$ of the (primal) SDP, and the equality involving $P_t$ follows from the complementary slackness conditions of the SDP.
This equality can be viewed as an approximate version of the Ricatti equation that applies to policies computed based on estimates of the system parameters (as opposed to the ``exact'' Ricatti equation, which is relevant only for \emph{optimal} policies of the actual LQR, that can only be computed based on the true parameters).

\subsection{Boundness of states} \label{sec:stability}

Next, we show that the policies computed by the algorithm keep the underlying system stable, and that state vectors visited by the algorithm are uniformly bounded with high probability.
To this end, consider the following sequence of ``good events'' $\calE_1 \supseteq \calE_2 \supseteq \cdots \supseteq \calE_T$, where for each $t$, 
\begin{align*}
    \calE_t
    =
    \biggl\{
    \forall \, s=1,\ldots,t, \quad
    \trace(\Delta_s V_s \Delta_s\tr) \leq 1
    ~,\;%\\[-0.75em]
    \norm{z_s}^2 \le 4 \kappa^4 e^{-\gamma(s-1)} \norm{x_{1}}^2 + \beta
    \biggr\}
    %~, \;\; \forall t \ge 1
    ~.
\end{align*}
That is, $\calE_t$ is the event on which everything worked as planned up to round $t$: our estimations were sufficiently accurate and the norms of $\set{z_s}_{s=1}^t$ were properly bounded.
% our estimations are sufficiently accurate up to round $t$, and the norms of $\set{z_s}_{s=1}^t$ are bounded.
%
We show that the events $\calE_1,\ldots,\calE_T$ hold with high probability; this would ensure that %$\sum_{s=1}^t \norm{z_s}^2 \le 2 \beta T$ and in turn guarantee that 
$V_t$ is appropriately bounded.

\begin{lemma} \label{lem:stabilityandboundedness}
Under the conditions of \cref{thm:main}, the event $\calE_T$ occurs with probability $\geq 1-\delta/2$.
\end{lemma}

\subsection{Sequential strong stability}

Crucially, \cref{lem:stabilityandboundedness} above holds true since the sequence of policies extracted by \cref{alg:main} from repeated solutions to the relaxed SDP is \emph{sequentially} strongly stable.
%A complete proof is given in \cref{sec:proofofseqstrongstab}.

\begin{lemma} \label{lem:seqstrong-main}
Assume the conditions of \cref{thm:main}, and further that
% Assume that %$\mu \ge 1 + 2\vartheta \norm{V_t}^{1/2}$ and $V_t \succeq 16 \kappa^{10} \mu I$ 
for any $t$,
$\spnorm{V_s} \le 4T$ for all $s=1,\ldots,t$.
Then the sequence of policies $K_1,\ldots,K_t$ 
% generated from the solutions of \eqref{eq:relaxed-sdp} for $s=1,\ldots,t$ 
is $(\kappa,\gamma)$-strongly stable for $\kappa = \sqrt{2\nu/\alpha_0 \sigma^2}$ and $\gamma = 1/2\kappa^2$.
\end{lemma}

This follows from a stability property of solutions to the relaxed SDP: we show that as the relaxed constraint becomes tighter, the optimal solutions of the SDP do not change by much (see \cref{sec:sdp-analysis}).
This, in turn, can be used to show that the policies extracted from these solutions are not drastically different from each other, and so the sequence of policies generated by the algorithm keeps the system stable.
\cref{lem:stabilityandboundedness} is then implied via a simple inductive argument: suppose that the state-vector norms are bounded up until round $t$; then the sequence of policies generated until time $t$ is strongly-stable thus keeping the norms of future states bounded with high probability.

% \tk{pass here}

% We remark that our analysis calls for \emph{sequential} strong stability as well as ensuring there are at most $\wt O(n)$ policy switches; both are required for the analysis to work.

% and even if we switch between policies only $O(\log{T})$ times

We remark that stability of the individual policies does not suffice, and the stronger \emph{sequential} strong stability condition is in fact required for our analysis.
Indeed, even if we guarantee the (non-sequential) strong stability of each individual policy, the system's state might blow up exponentially in the number of times the algorithm switches between policies: after switching to a new policy there is an initial burn-in period in which the norm of the state can increase by a constant factor (and thereafter stabilize).
Thus, even if we ensure that there are as few as $O(\log{T})$ policy switches, the states might become polynomially large in $T$ and deteriorate our regret guarantee.
Sequential strong stability wards off against such a blow up in the magnitude of states.

\subsection{Regret analysis}

Let us now connect the dots and sketch how our main result (\cref{thm:main}) is derived; for the formal proof, see \cref{sec:proof-thm:main}.
Consider the instantaneous regret $r_t = x_t\tr Q x_t + u_t\tr R u_t - J^\st$ and let $\wt R_T = \sum_{t=1}^{T} r_t \ind{\calE_t}$.
We will bound $\wt R_T$ with high probability, and since $R_t = \wt{R}_T$ with high probability due to \cref{lem:stabilityandboundedness}, this would imply a high-probability bound on $R_T$ from which the theorem would follow.

To bound the random variable $\wt R_T$, we appeal to \cref{lem:relaxed-sdp-main} that can be used to relate the instantaneous regret of the algorithm to properties of the SDP solutions it computes. 
Conditioned on the good event~$\calE_t$, the boundness of the visited states ensures that the confidence matrix $V_t$ is bounded as the lemma requires. 
The lemma then implies that
\begin{align*}
    Q + K_t\tr R K_t 
    \preceq
    P_t 
    - (\Ast+\Bst K_t)\tr P_t (\Ast + \Bst K_t) %\\
    %&\quad 
    + 2 \mu \trnorm{P_t} \IhatKt\tr V_t^{-1} \IhatKt
    ~.
\end{align*}
On the other hand, as $u_t = K_t x_t$ and $J^\st \geq \sigma^2 \trnorm{P_t}$ (which is also a consequence of \cref{lem:relaxed-sdp-main}), we have
\begin{align*}
    r_t
    =
    x_t\tr Q x_t + u_t\tr R u_t - J^\st
    \leq
    x_t\tr \Lr{Q + K_t\tr R K_t} x_t - \sigma^2 \trnorm{P_t}
    .
\end{align*}
Combining the inequalities and summing over $t=1,\ldots,T$, gives via some algebraic manipulations the following bound:
\begin{align}
    \wt R_T 
    \leq 
    & \phantom{+}
    \sum_{t=1}^T \Lr{x_t\tr P_t x_t - x_{t+1}\tr P_t x_{t+1}} \ind{\calE_{t}} 
    \nonumber \\
    & + 
    \sum_{t=1}^T w_t\tr P_t \Lr{\Ast + \Bst K_t} x_t \ind{\calE_t} 
    \nonumber \\
    & + 
    \sum_{t=1}^T \Lr{w_t\tr P_t w_t - \sigma^2 \trnorm{P_t}} \ind{\calE_t}
    \nonumber \\
    & + 
    \frac{4 \nu \mu}{\sigma^2} \sum_{t=1}^T \Lr{ z_t\tr V_t^{-1} z_t } \ind{\calE_t}
    ~. \label{eq:regrettildebound}
\end{align}
We now proceed to bounding each of the sums in the above.
The first sum above telescopes over consecutive rounds in which \cref{alg:main} uses the same policy and thus the matrix $P_t$ remains unchanged. 
Therefore, the number of remaining terms, each of which is bounded by a constant, is exactly the number of times that \cref{alg:main} computes a new policy.
We show that when the good events occur, the number of policy switches is at most $O(n\log T)$, which gives rise to the following.

\begin{lemma} \label{lem:regretterm1}
It holds that
\begin{align*}
    \sum_{t=1}^T \Lr{x_t\tr P_{t} x_t - x_{t+1}\tr P_{t} x_{t+1}} \ind{\calE_{t}}
    \leq
    \frac{4 \nu}{\sigma^2} \Lr{ 4\kappa^4 \norm{x_1}^2 + \beta } n \log T
    ~.
\end{align*}
\end{lemma}

The next two terms in the bound above are sums of martingale difference sequences, as the noise terms $w_t$ are i.i.d., and each $w_t$ is independent of $P_t$, $K_t$ and $x_t$.
Using standard concentration arguments, we show that both are bounded by $\tO(\sqrt{T})$ with high probability.

\begin{lemma}\label{lem:regretterm2}
With probability at least $1-\delta/4$, it holds that
\begin{align*}
    \sum_{t=1}^T w_t\tr P_{t} \Lr{\Ast + \Bst K_t} x_t \ind{\calE_t}
    \leq
    \frac{\nu \vartheta}{\sigma} \sqrt{3 \beta T \log\frac{4}{\delta}}
    ~.
\end{align*}
\end{lemma}

\begin{lemma} \label{lem:regretterm3}
With probability at least $1-\delta/4$, it holds that
\begin{align*}
    \sum_{t=1}^T \Lr{ w_t\tr P_t w_t - \sigma^2 \trnorm{P_t} } \ind{\calE_t}
    \leq
    8 \nu \sqrt{T \log^3\frac{4T}{\delta}}
    ~.
\end{align*}
\end{lemma}

Finally, using the elementary identity
$
    z\tr V^{-1} z
    \leq
    2 \log(\det(V + z z\tr)/\det(V))
$
for $V\succ 0$ and any vector $z$ such that $z\tr V^{-1} z \le 1$, we show that the final sum in the bound telescopes and can be bounded in terms of $\log(\det(V_{T+1})/\det(V_1))$;
in turn, the latter quantity can be bounded by $\O(n \log{T})$ using the fact that the $z_t$ are uniformly bounded on the event $\calE_T$.
This argument results with:

\begin{lemma} \label{lem:regretterm4}
We have
$
    \sum_{t=1}^T (z_t\tr V_t^{-1} z_t) \ind{\calE_t}
    \leq
    4 \beta n \log T
    %~.
$.
\end{lemma}

Our main theorem now follows by plugging-in the bounds into \cref{eq:regrettildebound}, using a union bound to bound the failure probability, and applying some algebraic simplification.

\section{The relaxed SDP program}
\label{sec:sdp-analysis}

In this section we present useful properties of the relaxed SDP program repeatedly solved by \cref{alg:main}, which are used to prove \cref{lem:relaxed-sdp-main,lem:seqstrong-main} discussed above and are central to our development.

The relaxed SDP program takes the following form.
Let $\mu > 0$ be a fixed parameter, and assume $A$, $B$ and $V$ are matrices such that the error matrix $\Delta = (A \, B) - (\Ast \, \Bst)$ satisfies $\trace(\Delta V \Delta\tr) \leq 1$.
\begin{alignat}{2}
    & \text{minimize} \quad && 
    \Sigma \bullet \qrmatrix
    \nonumber \\
    & \text{subject to} && 
    \Sigma_{xx} 
    \succeq 
    (A \; B) \Sigma (A \; B)\tr + W - \mu (\Sigma \bullet V^{-1}) I
    ,
    \label{eq:relaxed-sdp}
    \\
    & && \Sigma \succeq 0, \, \Sigma \in \reals^{n \times n}
    \nonumber
    ~.
\end{alignat}
%
% In \cref{lem:sdpvaluelb} we show that the relaxed program is ``optimistic'', in the sense that its value is at most $J^\st$---the value of the exact SDP \eqref{eq:lqr-sdp}.
% %
% We further present a policy defined by $K = \Sigma_{ux} \Sigma_{xx}^{-1}$ associated with a solution $\Sigma$ to the (primal) relaxed SDP. In \cref{lem:optimsticsdppolicy}, we utilize complementary slackness conditions of the SDP to show that this policy satisfies a Ricatti-like equation with respect to the dual solution $P$. These equations are crucial for bounding the regret of \cref{alg:main}. 
% %
% Finally, in \cref{sec:stability} we prove that the policies generated by repeated solutions to the relaxed program are stable in a very strong sense. This implies that even though \cref{alg:main} plays different policies at different rounds, the norms of $x_1,\ldots,x_T$ remain controlled.
%
For this section, the dual program to \eqref{eq:relaxed-sdp} will be useful:
\begin{alignat}{2}
    & \text{maximize} \quad && 
    P \bullet W
    \nonumber \\
    & \text{subject to} \quad && 
    \begin{psmallmatrix} Q - P & 0 \\ 0 & R \end{psmallmatrix} + (A \, B)\tr P (A \, B)
    \succeq
    \mu \trnorm{P} V^{-1}
    ,
    \label{eq:relaxed-dual-sdp}
    \\
    & && P \succeq 0, \, P \in \reals^{d \times d}
    \nonumber
    ~.
\end{alignat}

We now aim at proving \cref{lem:relaxed-sdp-main} which states that SDP~\eqref{eq:relaxed-sdp} is a relaxation of the original exact SDP~\eqref{eq:lqr-sdp}.
It follows directly from \cref{lem:sdpvaluelb,lem:optimsticsdppolicy} given below; see \cref{sec:relaxed-sdp-main-proof}.
First, we present a matrix-perturbation lemma also proven in \cref{sec:sigmaboundproof}.

\begin{lemma} \label{lem:sigmabound}
Let $X$ and $\Delta$ be matrices of matching sizes and assume $\Delta\tr \Delta \preceq V^{-1}$ for some matrix $V \succ 0$.
Then for any $\Sigma \succeq 0$ and $\mu \ge 1 + 2 \spnorm{X} \spnorm{V}^{1/2}$,
\[
    \spnorm{(X+\Delta) \Sigma (X+\Delta)\tr - X \Sigma X\tr}
    \leq
    \mu \Sigma \bullet V^{-1}
    ~.
\]
\end{lemma}

\begin{lemma} \label{lem:sdpvaluelb}
Assume $\mu \ge 1 + 2 \vartheta \spnorm{V}^{1/2}$.
Then the optimal value of SDP~\eqref{eq:relaxed-sdp} is at most $J^\st$.
Consequently, for a primal-dual optimal solution $\Sigma$, $P$ we have $\trnorm{\Sigma} \le J^\st/\alpha_0$ and $\trnorm{P} \le J^\st/\sigma^2$.
\end{lemma}

\begin{proof}
It suffices to show that $\Sigma^\st$, the solution to the original SDP~\eqref{eq:lqr-sdp}, is feasible for the relaxed SDP.
Indeed, $\Sigma^\st \succeq 0$, and combining \cref{eq:lqr-sdp,lem:sigmabound} (note that $\trace(\Delta V \Delta\tr) \le 1$ implies that $\Delta\tr \Delta \preceq V^{-1}$) yields \cref{eq:relaxed-sdp} due to 
\begin{align*}
    \Sigma_{xx}^\st 
    = 
    \Lr{\Ast \; \Bst} \Sigma^\st \Lr{\Ast \; \Bst}\tr + W %\\
    %&
    \succeq 
    \Lr{A \; B} \Sigma^\st \Lr{A \; B}\tr 
    +
    W 
    - 
    \mu \Lr{\Sigma^\st \bullet V^{-1}} I~.
\end{align*}
Therefore, it is feasible for SDP~\eqref{eq:relaxed-sdp}.
\end{proof}
The next lemma shows how to extract a policy from the relaxed SDP. 
Somewhat surprisingly, this policy is deterministic and has the linear form $x \mapsto K x$, as is the case in the original SDP.

\begin{lemma} \label{lem:optimsticsdppolicy}
Assume that 
$
    V \succeq (\nu \mu / \alpha_0 \sigma^2) I,
$
and
$
    \mu \ge 1 + 2 \vartheta \spnorm{V}^{1/2}.
$
Let $\Sigma$ and $P$ be primal and dual optimal solutions to the relaxed SDP. 
Then $\Sigma_{xx}$ is invertible, and for $K = \Sigma_{ux} \Sigma_{xx}^{-1}$ we have
\begin{align*}
    P
    =
    Q + K\tr P K + (A+BK)\tr P (A+BK)
    - \mu \trnorm{P} \IhatK V^{-1} \IhatK\tr
    .
\end{align*}
\end{lemma}

\begin{proof}
    Denote
    \[
        Z 
        = 
        \begin{psmallmatrix} 
            Q - P & 0 
            \\ 0 & R 
        \end{psmallmatrix} + \Lr{A \; B}\tr P \Lr{A \; B}  - \mu \trnorm{P} V^{-1}~.
    \]
    Recall the complementary-slackness conditions of the SDP, that read
    $
        \Sigma Z
        = 
        0
    $.
    We now show that $\Sigma_{xx} \succ 0$ and $\rank(\Sigma) = d$ as this would entail that 
    $$
        \Sigma = \IhatK \Sigma_{xx} \IhatK\tr 
        \quad\text{for}\quad 
        K = \Sigma_{ux} \Sigma_{xx}^{-1}
        .
    $$
    Thus the span of $\Sigma$ is the span of $\IhatK$ whence
    $
        \IhatK\tr 
        Z
        \IhatK
        = 
        0
    $ as required.
    
    To that end, we begin by stating the following basic fact about matrices:  
    For any two $n$-dimensional symmetric matrices, $X, Y$, that satisfy $X Y = 0$, it must be that $\rank(X) + \rank(Y) \le n$. 
    Then it suffices to show $\Sigma_{xx} \succ 0$ and 
    $
        \rank(Z)
        \ge 
        k
    $.   
    %    
    %Indeed, \cref{lem:sdpvaluelb}, together with $J^\st \le \nu$ and $\qrmatrix \succeq \alpha_0 I$, gives     $\trnorm{\Sigma} \le \nu / \alpha_0$ and $\trnorm{P} \le \nu / \sigma^2$.
    %
    Indeed, using \cref{lem:sdpvaluelb},
    % {\small
    \begin{align} \label{eq:primalbound}
        &\mu \Lr{\Sigma \bullet V^{-1}} I
        \preceq
        \mu \trnorm{\Sigma} \spnorm{V^{-1}} I
        \prec
        \mu \, \frac{\nu}{\alpha_0} \, \frac{\alpha_0 \sigma^2}{\nu \mu} I
        = 
        W
        , \\ \label{eq:dualbound}
        &\mu \trnorm{P} V^{-1}
        \preceq 
        \mu \trnorm{P} \spnorm{V^{-1}} I
        \prec
        \mu \, \frac{\nu}{\sigma^2} \, \frac{\alpha_0 \sigma^2}{\nu \mu} I
        \preceq 
        \qrmatrix
        ,
    \end{align}
    % }
    as $W = \sigma^2 I$ and $\qrmatrix \succeq \alpha_0 I$.
    Plugging \cref{eq:primalbound} into \cref{eq:relaxed-sdp} and using $\Sigma \succeq 0$, shows that $\Sigma_{xx} \succ 0$. 
    Moreover, 
    $Z$ is the difference of
    $$
        \qrmatrix + \Lr{A \; B}\tr P \Lr{A \; B} - \mu \trnorm{P} V^{-1}
        ,
    $$ 
    which is of rank $d+k$ in light of \cref{eq:dualbound} and since $P \succeq 0$, and
    $
        \begin{psmallmatrix}
        P & 0 \\
        0 & 0
        \end{psmallmatrix}
    $
    which is of rank at most $d$.
    Therefore, $\rank(Z) \ge k$ as required.
\end{proof}

We continue with proving the main result of this section that would imply \cref{lem:seqstrong-main} (see \cref{sec:seqstrong-main-proof}).
We show that the sequence of policies generated by solving a certain series of relaxed SDPs is strongly-stable.
%A complete proof is given in \cref{sec:proofofseqstrongstab}.

\begin{theorem} \label{thm:seqstrongstability}
Let ${P}_1,{P}_2,\ldots$ be optimal solutions to the relaxed dual SDP; each ${P}_t$ associated with $(A_t \; B_t)$ and $V_t$ respectively. 
Let $\kappa = \sqrt{2\nu/\alpha_0 \sigma^2}$, $\gamma = 1/2\kappa^2$, and suppose that $\mu \ge 1 + 2\vartheta \norm{V_t}^{1/2}$ and $V_t \succeq 16 \kappa^{10} \mu I$ 
for all $t$. 
Moreover, let $K_t$ be the policy associated with $P_t$ (as in \cref{lem:optimsticsdppolicy}).
Then the sequence $K_1,K_2,\ldots$ is $(\kappa,\gamma)$-strongly stable.
\end{theorem}

The proof is given by combining the following two lemmas.
Indeed, in \cref{sec:proofofstrongstab} we show that each policy $K_t$ is strongly stable.

\begin{lemma} \label{lem:strongstability}
$K_t$ is $(\kappa,\gamma)$-strongly stable for $\Ast + \Bst K_t = H_t L_t H_t^{-1}$ where $H_t = P_t^{1/2}$ and $\spnorm{L_t} \le 1-\gamma$. 
Moreover, $(\alpha_0/2) I \preceq P_t \preceq (\nu / \sigma^2) I$.
\end{lemma}

Furthermore, having established strong stability, the next lemma shows that $P_t$ is ``close'' to $P_{t+1}$ (see \cref{sec:proofofpclosetopst} for a proof).

\begin{lemma} \label{lem:phatclosetopstar}
$
    P_t 
    \preceq 
    P^\st
    \preceq
    P_{t+1} + (\alpha_0 \gamma/2) I
$
for all $t \ge 1$.
\end{lemma}

\begin{proof}[Proof of \cref{thm:seqstrongstability}]
We show that the conditions for sequential strong-stability hold.
Notice that not only does \cref{lem:strongstability} show that for all $t$, $K_t$ is $(\kappa,\gamma)$-strongly stable, it also gives us uniform upper and lower bounds on $H_t = P_t^{1/2}$ as 
$
    \spnorm{P_t} \le \trnorm{P_t} \leq \nu / \sigma^2
$ (\cref{lem:sdpvaluelb}), and 
    $\spnorm{P_t^{-1}} \le 2 / \alpha_0
$.
Together with $P_{t+1} \succeq (\alpha_0/2) I$, the lemma implies
\begin{align*}
    \spnorm{H_{t+1}^{-1}H_t}^2
    &=
    \spnorm{P_{t+1}^{-1/2} P_t^{1/2}}^2 \\
    &=
    \spnorm{P_{t+1}^{-1/2} P_t P_{t+1}^{-1/2}} \\
    &\leq
    \spnorm{I + \thalf \alpha_0 \gamma P_{t+1}^{-1}} \\
    &\leq
    1 + \thalf \alpha_0 \gamma \spnorm{P_{t+1}^{-1}} \\
    &\leq
    1+\gamma
    ~.
\end{align*}
Thus
$
    \spnorm{H_{t+1}^{-1}H_t} 
    \le 
    \sqrt{1+\gamma} 
    \le 
    1 + \thalf\gamma
$
which provides sequential strong-stability.
\end{proof}

\section{Warm-up Using a Stable Policy}
\label{sec:warmstart}

In this section we give a simple warm-up scheme that can be used in an initial exploration phase, after which the conditions of our main algorithm are met.
Here we assume that we are given a policy $K_{\ws}$ which is known to be $(\kappa_{\ws},\gamma_{\ws})$-strongly stable for the LQR \eqref{eq:lqr}.

Starting from $x_1 = 0$ and over $T_{\ws}$ rounds, the warm-up procedure samples actions $u_t \sim \mathcal{N} (K_{\ws} x_t, 2 \sigma^2 \kappa_{\ws}^2 I)$ independently; this is summarized in \cref{alg:warmstart}.

\begin{algorithm}[h]
\caption{Warm-up procedure}
\begin{algorithmic}
    \label{alg:warmstart}
    \STATE {\bf input}: $(\kappa_{\ws},\gamma_{\ws})$-strongly stable policy $K_{\ws}$, horizon $T_{\ws}$.
    \FOR{$t=1,\ldots,T_{\ws}$}
        \STATE {\bf observe} state $x_{t}$.
        \STATE {\bf play} $u_t \sim \calN(K_{\ws} x_t, 2 \sigma^2 \kappa_{\ws}^2 I)$.
    \ENDFOR
\end{algorithmic}
\end{algorithm}

Let
$
    V_{\ws} 
    = 
    \sum_{t=1}^{T_{\ws}} z_t z_t\tr
$
be the empirical covariance matrix corresponding to the samples $z_t$ collected during warm-up, where $z_t = \begin{psmallmatrix} x_t \\ u_t \end{psmallmatrix}$ for all $t$.
The main result of this section gives upper and lower bounds on the matrix $V_\ws$.

\begin{theorem} \label{thm:Vws}
Let $\delta \in (0,1)$.
Provided that 
%$T_{\ws} \ge 400 (n + \log(2/\delta))$
$
    T_{\ws} \ge \poly(\sigma, n, \vartheta, \kappa_{\ws},\gamma_{\ws}^{-1}, \log(\delta^{-1}))
$
, we have with probability at least $1-\delta$ that
\begin{align*}
    \trace(V_{\ws}) 
    &\leq 
    T_\ws \cdot \frac{300 \sigma^2 \kappa_{\ws}^4}{\gamma_{\ws}^2} \Lr{n + k \vartheta^2 \kappa_\ws^2} \log\frac{T_{\ws}}{\delta}
    ~, \\
    \norm{x_{T_{\ws}}}^2 
    &\le
    \frac{150 \sigma^2 \kappa_{\ws}^2}{\gamma_{\ws}} \Lr{n + k \vartheta^2 \kappa_\ws^2} \log\frac{T_{\ws}}{\delta}~, \\
    V_\ws
    &\succeq 
    \frac{T_{\ws} \sigma^2}{80} I~,
\end{align*}
and for $V = V_{\ws} + \sigma^2 \vartheta^{-2} I$ and initial estimates
$
    \Lr{A_{\ws} \, B_{\ws}} = \sum_{t=1}^{T_{\ws}-1} x_{t+1} z_t\tr V^{-1}
$
we have
\[
    \trace(\Delta_{\ws} V \Delta_{\ws}\tr) 
    \le 
    20 n^2 \sigma^2 \log \frac{T_{\ws}}{\delta}~.
\]
\end{theorem}

With \cref{thm:Vws} in hand, the proof of \cref{corr:regretplusws} readily follows; see details in \cref{sec:regretpluswsproof}.
The proof of \cref{thm:Vws} itself is based on adaptations of techniques developed in \citet{simchowitz2018learning} in the context of identification of Linear Dynamical Systems, and is given in \cref{sec:warmstartproof}.

\subsection*{Acknowledgments}

We thank Yoram Singer and Kunal Talwar for many helpful discussions.
YM was supported in part by a grant from the Israel Science Foundation (ISF).
AC thanks Lotem Peled for her assistance and support.

% \newpage

% \cite{*}
\setlength{\bibsep}{1.5ex plus 0.3ex}
\bibliographystyle{plainnat}
\bibliography{lqr}

% \clearpage
% \onecolumn

\appendix
% {\noindent\LARGE\bfseries Supplementary Material}

\section{Preliminaries}

\subsection{Concentration inequalities} 
\label{sec:concentration}

First, we state a variant of the Hanson-Wright inequality \citep{hanson1971bound,wright1973bound}, which can be found in \cite{hsu2012tail}.

\begin{theorem}[Hanson-Wright inequality] \label{thm:hanson}
Let $x \sim \mathcal{N}(0,I_n)$ be a Gaussian random vector and let $A \in \reals^{m \times n}$.
For all $z > 0$,
\begin{align*}
    \Pr\Lrbra{ \norm{Ax}^2 - \norm{A}_\frob^2 > 2 \norm{A}_\frob \spnorm{A} \sqrt{z} + 2 \norm{A}^2 z }
    <
    e^{-z}
    .
\end{align*}
In particular, if $x \sim \mathcal{N}(0,\Sigma)$ then $\bbE\norm{Ax}^2 = \trace(A \Sigma A\tr)$ and for any $z \geq 1$,
\begin{align*}
    \Pr\Lrbra{ \norm{Ax}^2 - \bbE\norm{Ax}^2 > 4z \norm{\Sigma} \spnorm{A} \norm{A}_\frob }
    <
    e^{-z}
    .
\end{align*}
\end{theorem}

The following is Azuma's inequality for concentration of martingales with bounded differences.

\begin{theorem}[\citealp{azuma1967weighted}]
    \label{thm:azuma}
    Let $X_1,\ldots,X_N$ be a martingale difference sequence such that
    $\lvert X_i \rvert \le c$ for all $i=1,\ldots,n$. 
    Then,
    \[
        \Pr \lrbra{ \sum_{i=1}^n X_i > t }
        \leq 
        \exp \lr{-\frac{t^2}{2 n c^2}}
        ~.
    \]
\end{theorem}

The following is is a self-normalized concentration inequality for vector-valued martingales useful for guaranteeing generalization in linear regression.

\begin{theorem}[\citealp{abbasi2011improved}]
\label{thm:regressionconcentration}
Let $(\calF_t)_{t=0}^\infty$ be a filtration and let $(\eta_t)_{t=1}^\infty$ be a real-valued martingale difference sequence adapted to $(\calF_t)$ such that $\eta_{t}$ is $R$-sub-Gaussian conditioned on $\calF_{t-1}$, that is,
\[
    \bbE \big[ e^{\lambda \eta_t} \bigm\vert \calF_{t-1} \big] 
    \leq 
    e^{\lambda^2 R^2 / 2}
    ,
    \qquad
    \forall ~ t \geq 1
    ~.
\]
Further, let $(u_t)_{t=1}^\infty$ be an $\bbR^n$-valued stochastic process adapted to $(\calF_{t-1})_{t=1}^\infty$, let $V \in \reals^{n \times n}$ positive definite matrix, and define
\[
    U_t = \sum_{s=1}^{t-1} \eta_s u_s
    ~,
    \qquad
    V_t 
    = 
    V + \sum_{s=1}^{t-1} u_{s} u_{s}\tr~,
    \qquad 
    t=1,2,\ldots
    .
\]
Then, for any $\delta \in (0,1)$ we have with probability at least $1-\delta$ that
\[
    U_t\tr V_t^{-1} U_t 
    \le
    2R^2 \log \LR{\frac{1}{\delta} \frac{\det(V_t)}{\det(V)}}
    \qquad
    \forall ~ t=1,2,\ldots 
    ~.
\]
\end{theorem}

\subsection{Technical Lemmas}

\begin{lemma} \label{lem:semidefinitebound}
Let $X$ and $\Delta$ be matrices of matching sizes and assume $\Delta\tr \Delta \preceq V^{-1}$ for some matrix $V \succ 0$. Then for any $P \succeq 0$ and 
$\mu \ge 1 + 2 \spnorm{X} \spnorm{V}^{1/2}$,
\[
    -\mu \trnorm{P} V^{-1}
    ~\preceq~
    (X + \Delta)\tr P (X+\Delta) - X\tr P X 
    ~\preceq~ 
    \mu \trnorm{P} V^{-1}
    ~.
\]
\end{lemma}

\begin{proof}
Note that $(X + \Delta)\tr P (X+\Delta) - X\tr P X = X\tr P \Delta + \Delta\tr P X + \Delta\tr P \Delta$. 
Let $\epsilon > 0$. We have
\begin{align*}
    X\tr P \Delta + \Delta\tr P X
    \preceq
    \epsilon^{-1} X\tr P X + \epsilon \Delta\tr P \Delta 
    ;
\end{align*}
this can be seen by expanding the inequality
$
    (\epsilon^{-1/2} X - \epsilon^{1/2} \Delta)\tr P (\epsilon^{-1/2} X - \epsilon^{1/2} \Delta)
    \succeq
    0
    .
$
Setting $\epsilon = \spnorm{X} \spnorm{V}^{1/2}$ and using our assumption that $\Delta\tr \Delta \preceq V^{-1}$ yields
\begin{align*}
    X\tr P \Delta + \Delta\tr P X
    &\preceq 
    \epsilon^{-1} X\tr P X + \epsilon \Delta\tr P \Delta
    \\
    &\preceq 
    \epsilon^{-1} \spnorm{X}^2 \spnorm{P} I + \epsilon 
    \spnorm{P} V^{-1}
    \\
    &= 
    \spnorm{X} \spnorm{P} \lr{ \spnorm{V}^{-1/2} I + 
    \spnorm{V}^{1/2} V^{-1}} 
    \tag{$\epsilon = \spnorm{X} \spnorm{V}^{1/2}$} \\
    &\preceq 
    \spnorm{X} \spnorm{P} \spnorm{V}^{1/2} V^{-1}~.
    \tag{$\spnorm{V}^{-1/2} I \preceq \spnorm{V}^{1/2} V^{-1}$}
\end{align*}
This, together with $\Delta\tr P \Delta \preceq \spnorm{P} \Delta\tr \Delta \preceq \spnorm{P} V^{-1}$,
proves one direction of the inequality.
For the other direction, a similar argument shows 
\begin{align*}
    (X + \Delta)\tr P (X+\Delta) - X\tr P X 
    \succeq 
    X\tr P \Delta + \Delta\tr P X 
    % \\
    \succeq 
    -\spnorm{P} \spnorm{X} \spnorm{V}^{1/2} V^{-1}~
    &.\qedhere
\end{align*}
\end{proof}

\begin{lemma} \label{lem:lypstonglystable}
Let $X, Z$ be symmetric matrices of equal sizes and $Y$ a $(\kappa,\gamma)$-strongly stable matrix such that
$
    X 
    \preceq 
    Y\tr X Y + Z
$.
Then
$
    X 
    \preceq 
    (\kappa^2/\gamma) \spnorm{Z} I
$.
\end{lemma}

\begin{proof}
The inequality $X \preceq Y\tr X Y + Z$ implies there exists a matrix $M$ such that $M \succeq 0$, and $X = Y\tr X Y + Z - M$. As $Y$ is stable, the equation has a unique solution that satisfies:
\[
    X 
    = 
    \sum_{t=0}^\infty (Y^t)\tr (Z - M) Y^t 
    \preceq 
    \sum_{t=0}^\infty (Y^t)\tr Z Y^t
    .
\]
Let us proceed in bounding the norm of the right-hand side of this inequality.
As $Y$ is $(\kappa,\gamma)$-strongly stable, $Y = H L H^{-1}$ with $\spnorm{L} \le 1-\gamma$ and $\spnorm{H}\spnorm{H^{-1}} \le \kappa$. Therefore,
\[
\LRnorm{\sum_{t=0}^\infty (Y^t)\tr Z Y^t} 
\le \sum_{t=0}^\infty \spnorm{Y^t}^2 \spnorm{Z}~,
\]
and we have
\[
\spnorm{Y^t} 
= \spnorm{H L^t H^{-1}} 
\le \spnorm{H}\spnorm{H^{-1}} \spnorm{L} 
\le \kappa (1-\gamma)~.
\]
This implies that
\begin{align*}
    \LRnorm{\sum_{t=0}^\infty (Y^t)\tr Z Y^t} 
    \le \spnorm{Z} \sum_{t=0}^\infty \kappa^2 (1-\gamma)^{2t}
    \le \spnorm{Z} \kappa^2 \sum_{t=0}^\infty (1-\gamma)^{t}
    = \frac{\kappa^2}{\gamma} \spnorm{Z}
    &.\qedhere
\end{align*}
\end{proof}

\begin{lemma} \label{lem:det1}
For $M\succ 0$ and a vector $z$ such that $z\tr M^{-1} z \le 1$, 
\begin{align*}
    z\tr M^{-1} z
    \leq
    2 \log\frac{\det(M + z z\tr)}{\det M}
    ~.
\end{align*}
\end{lemma}

\begin{proof}
Observe that 
$
    \det(M + z z\tr)
    = 
    \det(M) \det(I + M^{-1/2} z z\tr M^{-1/2} )
    =
    (1 + z\tr M^{-1} z) \det(M)
$ 
by the determinant lemma, and so
\[
    \log(1+ z\tr M^{-1} z) 
    = 
    \log \frac{\det(M + z z\tr)}{\det(M)}
    ~.
\]
The proof is finished using the concavity of $x \mapsto \log(1+x)$ and the fact that $0 \leq z\tr M^{-1} z \leq 1$:
\begin{align*}
    \log (1+z\tr M^{-1} z)
    \geq
    (1 - z\tr M^{-1} z) \log 1 + (z\tr M^{-1} z) \log 2
    \geq
    \thalf z\tr M^{-1} z
    ~.&\qedhere
\end{align*}
\end{proof}

\begin{lemma} \label{lem:det2}
If $N \succeq M \succ 0$, then for any vector $v$ one has
\begin{align*}
    v\tr N v
    \leq
    \frac{\det{N}}{\det{M}} \,
    v\tr M v
    .
\end{align*}
\end{lemma}

\begin{proof}
Note that the claimed inequality is equivalent to $N \preceq (\det(N)/\det(M)) M$, which in turn is equivalent to $\spnorm{M^{-1/2} N M^{-1/2}} \leq \det(M^{-1/2} N M^{-1/2})$.
The latter is true because $R = M^{-1/2} N M^{-1/2} \succeq I$, and so the product of the eigenvalues of $R$ (all of which are $\geq 1$) is no smaller than the maximal eigenvalue of $R$.
\end{proof}

\subsection{Proof of \cref{lem:strg-stab-norm}}
\label{sec:strg-stab-norm-proof}

\begin{proof}
Following the sequence $K_1,K_2,\ldots$ induces updates $x_{t+1} = (\Ast + \Bst K_t) x_t + w_t$. 
Thus
\begin{align*}
    x_t 
    = 
    M_1 x_1 + \sum_{s=1}^{t-1} M_{s+1} w_s
    ~,
\end{align*}
where
\begin{align*}
    M_t = I;
    \quad
    M_s = M_{s+1}\Lr{\Ast+\Bst K_s}= \prod_{j=s}^{t-1} (\Ast+\Bst K_j)
    ,
    \;\;\;
    \forall \; 1 \leq s \leq t-1
    .
\end{align*}

Since the sequence $K_1,K_2,\ldots$ is sequential strong stable, there exist matrices $H_1, H_2, \ldots$ and $L_1, L_2, \ldots$ such that $\Ast+\Bst K_j = H_j L_j H_j^{-1}$ with the properties specified in \cref{def:seq-strong-stab}.
Thus, we have for all $1 \le s < t$ that
\begin{align*}
    \spnorm{M_s}
    &=
    \LRnorm{\prod_{j=s}^{t-1} H_j L_j\tr H_j^{-1}}
    \\
    &\leq
    \spnorm{H_{t-1}} \, \LR{\prod_{j=s}^{t-1} \spnorm{L_j}} \, \LR{\prod_{j=s}^{t-2} \spnorm{H_{j+1}^{-1} H_j}} \, \spnorm{H_{s}^{-1}}
    \\
    &\leq 
    B_0 (1-\gamma)^{t-s} (1+\gamma/2)^{t-s} (1/b_0)
    \\
    &\leq
    \kappa (1-\gamma/2)^{t-s}
    ~.
\end{align*}
As $\kappa \ge 1$, the same holds for $M_t$.

Thus, for all $t \ge 1$,
\begin{align*}
    \norm{x_t}
    &\leq
    \spnorm{M_1} \norm{x_1} + \sum_{s=1}^{t-1} \spnorm{M_{s+1}} \norm{w_s}
    \\
    &\leq
    \kappa (1-\gamma/2)^{t-1} \norm{x_1} + \kappa \sum_{s=1}^{t-1} (1-\gamma/2)^{t-s-1} \norm{w_s}
    \\
    &\leq
    \kappa e^{-\gamma (t-1)/2} \norm{x_1} + \kappa \max_{1 \leq s < t} \norm{w_s} \sum_{t=1}^{\infty} (1-\gamma/2)^{t}
    \\
    &=
    \kappa e^{-\gamma (t-1)/2} \norm{x_1} + \frac{2\kappa}{\gamma} \max_{1 \leq s < t} \norm{w_s}
    ~.
    \qedhere
\end{align*}
\end{proof}

\section{Proofs of \cref{sec:analysis}}

For the proofs in this section, we require the following two simple lemmas.
% The first bounds the number of policy switches the algorithm makes assuming the states are properly bounded.

\begin{lemma} \label{lem:boundnumofswitches}
% Assume $T \ge 2$, $\lambda \ge 1$ and $\sum_{s=1}^t \norm{z_s}^2 \le 2 \beta T$.
% The number of policy switches \cref{alg:main} makes is at most $2n \log{T}$.
Assume $T \ge 2$, $\lambda \ge 1$ and $\sum_{s=1}^t \norm{z_s}^2 \le 2 \beta t$. Let $V_t = \lambda I + \beta^{-1} \sum_{t=1}^{t-1} z_t z_t\tr$. Then
$$
    \log\frac{\det(V_t)}{\det(V_1)}
    \leq 
    2n \log T
$$.
\end{lemma}

\begin{proof}
We have
\begin{align*}
    \log\frac{\det(V_t)}{\det(V_1)}
    &=
    \log \det (V_1^{-1/2} V_t V_1^{-1/2}) \\
    &\le 
    n \log \spnorm{V_1^{-1/2} V_t V_1^{-1/2}} 
    \\
    &\leq
    n \log \LR{1 + \frac{1}{\beta \lambda} \sum_{s=1}^{t-1} \norm{z_s}^2} 
    \tag{$V_1 = \lambda I$; $V_t = \beta^{-1} \sum_{s=0}^{t-1} z_s z_s\tr + V_1$} \\
    &\leq
    n \log (1 + 2 T)
    \tag{$\sum_{s=1}^t \norm{z_s}^2 \le 2 \beta T$; $\lambda \ge 1$} 
    \\
    &\leq 
    2n \log T
    \tag{$T \ge 2$}
    ~.
\end{align*}
\end{proof}

\begin{lemma} \label{lem:sumofnormssq}
Assume that $\norm{z_s}^2 \le 4 \kappa^4 e^{-\gamma(t-1)} + \beta$ for $s=1,\ldots,t$, and $\kappa = \sqrt{2 \nu / \alpha_0 \sigma^2}$, $\gamma = 1/2\kappa^2$. 
Also suppose that $t \ge \norm{x_1}^2$. 
Then $\sum_{s=1}^t \norm{z_s}^2 \le 2 \beta t$.
\end{lemma}

\begin{proof}
\begin{align*}
    \sum_{s=1}^t \norm{z_s}^2
    \le
    \frac{8 \kappa^4}{\gamma} \norm{x_1}^2 + \beta t
    = 
    16 \kappa^6 \norm{x_1}^2 + \beta t
    \le
    2 \beta t
    ~.
\end{align*}
\end{proof}

\subsection{Proof of Main Theorem (\cref{thm:main})}
\label{sec:proof-thm:main}

\begin{proof}
Consider the instantaneous regret $r_t = x_t\tr Q x_t + u_t\tr R u_t - J^\st$ and let $\wt R_T = \sum_{t=1}^{T} r_t \ind{\calE_t}$.
We will bound $\wt R_T$ with high probability, and due to \cref{lem:stabilityandboundedness} this would imply a high-probability bound on $R_T$ from which the theorem would follow.

To bound the random variable $\wt R_T$, we appeal to \cref{lem:relaxed-sdp-main}. 
The lemma requires that, at round $s$, the confidence matrix $V_s$ is well-conditioned. 
Indeed, assuming $\calE_t$ holds, then on the one hand $V_s \succeq \lambda I$ for $\lambda \geq (10 \nu \vartheta / \alpha_0 \sigma^2) \sqrt{T}$, and on the other hand, $\spnorm{V_s} \leq \lambda + \beta^{-1} \sum_{r=1}^{s-1} \norm{z_r}^2 \leq T + 2 T \leq 4 T$ thanks to \cref{lem:sumofnormssq}.
Now, for any time $t$, let $\tau(t)$ denote the last time before round $t$ in which \cref{alg:main} updated its policy, so that $A_t = A_{\tau(t)}$, $B_t = B_{\tau(t)}$, $K_t = K_{\tau(t)}$ and $P_t = P_{\tau(t)}$ for all $t$. 
\cref{lem:relaxed-sdp-main} then implies
\[
    Q + K_t\tr R K_t 
    =
    P_t 
    - (A_t+B_t K_t)\tr P_t (A_t + B_t K_t) 
    + \mu \trnorm{P_t} \IhatKt\tr V_{\tau(t)} \IhatKt
    ~.
\]

On the other hand, as $u_t = K_t x_t$ and $J^\st \geq \sigma^2 \trnorm{P_t}$, we have
\begin{align*}
    r_t
    =
    x_t\tr Q x_t + u_t\tr R u_t - J^\st 
    \leq
    x_t\tr (Q + K_t\tr R K_t) x_t - \sigma^2 \trnorm{P_t}
    .
\end{align*}
Thus, given that $\calE_t$ holds, %plugging the above into the expression for $r_t$ yields
\begin{align*}
    r_t
    % &= 
    % x_t\tr Q x_t + u_t\tr R u_t - J^\st 
    % \\
    &\leq
    x_t\tr P_t x_t - x_t\tr \Lr{A_t + B_t K_t}\tr P_t \Lr{A_t + B_t K_t} x_t 
    + \mu \trnorm{P_t} z_t\tr V_{\tau(t)}^{-1} z_t 
    - \sigma^2 \trnorm{P_t} 
    \\
    &=
    x_t\tr P_t x_t - x_{t+1}\tr P_t x_{t+1} 
    \\
    &\quad + x_t\tr \Lr{\Ast + \Bst K_t}\tr P_t \Lr{\Ast + \Bst K_t} x_t - x_t\tr \Lr{A_t + B_t K_t}\tr P_t \Lr{A_t + B_t K_t} x_t
    \\
    &\quad + w_t\tr P_t \Lr{\Ast + \Bst K_t} x_t 
    \\
    &\quad + w_t\tr P_t w_t - \sigma^2 \trnorm{P_t}
    \\
    &\quad + 
    \mu \trnorm{P_t} (z_t\tr V_{\tau(t)}^{-1} z_t)
    ~.
\end{align*}
\cref{lem:semidefinitebound} now gives
\begin{align*}
    x_t\tr \Lr{\Ast + \Bst K_t}\tr P_t \Lr{\Ast + \Bst K_t} x_t - x_t\tr \Lr{A_t + B_t K_t}\tr P_t \Lr{A_t + B_t K_t} x_t
    &\leq 
    \mu \trnorm{P_t} z_t\tr V_{\tau(t)}^{-1} z_t
    ~,
\end{align*} 
and since \cref{alg:main} maintains that $\det(V_t) \leq 2 \det(V_{\tau(t)})$, we have $z_t\tr V_{\tau(t)}^{-1} z_t \le 2 z_t\tr V_t^{-1} z_t$ as a result of \cref{lem:det2}. 
This, along with the fact that $\trnorm{P_t} \le \nu / \sigma^2$ on $\calE_t$ (recall \cref{lem:relaxed-sdp-main}), yields
\begin{align*}
    \wt R_T 
    \leq 
    & \phantom{+}
    \sum_{t=1}^T \Lr{x_t\tr P_t x_t - x_{t+1}\tr P_t x_{t+1}} \ind{\calE_{t}} 
    \\
    & + 
    \sum_{t=1}^T w_t\tr P_t \Lr{\Ast + \Bst K_t} x_t \ind{\calE_t} 
    \\
    & + 
    \sum_{t=1}^T \Lr{w_t\tr P_t w_t - \sigma^2 \trnorm{P_t}} \ind{\calE_t}
    \\
    & + 
    \frac{4 \nu \mu}{\sigma^2} \sum_{t=1}^T \Lr{ z_t\tr V_t^{-1} z_t } \ind{\calE_t}
    ~. 
\end{align*}

The theorem now follows by plugging in the bounds of \cref{lem:regretterm1,lem:regretterm2,lem:regretterm3,lem:regretterm4}, using a union bound to bound the failure probability, and applying some algebraic simplifications.
\end{proof}

\subsection{Proof of \cref{corr:regretplusws}} \label{sec:regretpluswsproof}

\begin{proof}
First, let us show that if \cref{thm:Vws} holds then the initial conditions of \cref{thm:main} are satisfied.
Indeed, using $V \succeq V_{\ws} \succeq (T_{\ws}\sigma^2/80) I$ gives 
\[
    \norm{\Delta_0}_F
    \leq
    40n \sqrt{\frac{\log (T_{\ws} / \delta)}{T_{\ws}}}
    ~,
\]
which, by our choice of $T_{\ws}$, is at most $ \tfrac{\alpha_0^5 \sigma^{10}}{2^{13}\nu^5 \vartheta \sqrt{T}}$.
This means that the conditions of \cref{thm:main} hold.
Now, by a union bound, with probability at least $1-\delta$ \cref{thm:Vws,thm:main,lem:concentration} hold, each with probability at least $1-\delta/3$.
Then the regret of this procedure is 
\begin{align*}
    \sum_{t=1}^{T_{\ws}} \Lr{ x_t\tr Q x_t + u_t\tr R u_t - J^\st }
    &\le
    \qrmatrix \bullet V_{\ws} \\
    &\le
    \alpha_1 T_\ws \frac{300 \sigma^2 \kappa_{\ws}^4}{\gamma_{\ws}^2} \Lr{n + k \vartheta^2 \kappa_\ws^2} \log\frac{T_{\ws}}{\delta} \\
    &\le
    \frac{2^{35} \alpha_1 n^2 \nu^{5}\vartheta \kappa_{\ws}^4}{\alpha_0^{5} \sigma^{8} \gamma_{\ws}^2} 
    \Lr{n + k \vartheta^2 \kappa_\ws^2} \sqrt{T} \log^2 \frac{T}{\delta}~,
\end{align*}
and regret on the remaining rounds is bounded by virtue of \cref{thm:main}.
\end{proof}

\subsection{Proof of \cref{lem:concentration}} \label{sec:proofofconcentration}

\begin{proof}
Denote $\Theta_\st = \lr{\Ast \, \Bst}$ and $\Theta_t = \lr{A_t \, B_t}$. 
Note that the solution to the least-square estimate is given as:
\begin{equation} \label{eq:ls-estimate}
    \Theta_t 
    =
    \lr{\lambda \Theta_0 + \frac{1}{\beta} \sum_{s=1}^{t-1} x_{s+1} z_s\tr} V_t^{-1}
    ~.
\end{equation}
Plugging $x_{s+1} = \Theta_\st z_s + w_s$ into \cref{eq:ls-estimate} and denoting
$S_t = \sum_{s=1}^{t} w_s z_s\tr$, we have
\begin{align*}
	\Theta_{t} 
    = 
    \Theta_\st \cdot \frac{1}{\beta} \sum_{s=1}^{t-1} z_s z_s\tr V_t^{-1}
    + 
    \frac{1}{\beta} S_{t} V_t^{-1} 
    + 
    \lambda \Theta_0 V_t^{-1}
    = 
    \Theta_\st 
    + 
    \frac{1}{\beta} S_{t} V_t^{-1}
    +
    \lambda \Delta_0 V_t^{-1}
    ~,
\end{align*}
whence
\begin{align*}
    \trace(\Delta_t V_t \Delta_t\tr) 
    &\leq
    \frac{2}{\beta^2} \trace( S_{t} V_t^{-1} S_{t}\tr )
    + 
    2\lambda^2 \trace(\Delta_0 V_t^{-1} \Delta_0\tr) 
    \\
    &\leq
    \frac{2}{\beta^2} \trace( S_{t} V_t^{-1} S_{t}\tr)
    + 
    2 \lambda \norm{\Delta_0}_\frob^2
    .
    \tag{$\because ~ V_t \succeq \lambda I$}
\end{align*}
To get the result we need to bound the first term. 
Denote $S_t(i) = \sum_{s=1}^{t-1} w_s(i) z_s$ for all $i=1,\ldots,d$. 
For each $i$, applying \cref{thm:regressionconcentration} yields that, with probability at least $1-\delta/d$,
\begin{align*}
    S_t(i)\tr V_t^{-1} S_t(i) 
    &\le 
    2 \sigma^2 \beta \log \LR{\frac{d}{\delta} \frac{\det(V_t)}{\det(\beta V_1)}} ~.
\end{align*}
By additionally applying a union bound, the above holds with probability at least $1-\delta$ for all $i=1,\ldots,d$ simultaneously, and then
\[
    \trace(S_t V_t^{-1} S_t\tr)
    =
    \sum_{i=1}^d S_t(i)\tr V_t^{-1} S_t(i) 
    \leq
    2 \sigma^2 \beta d \log \LR{\frac{d}{\delta} \frac{\det(V_t)}{\det(\beta V_1)}}~.
\]
Plugging this to the inequality above, and using $\beta \ge 1$, gives the main statement of the lemma.

To show $\trace(\Delta_t V_t \Delta_t) \le 1$ under the conditions of \cref{alg:main}, note that $\norm{\Delta_0}_F^2 \le \epsilon \le 1/4\lambda$ by assumption. Thus it remains to prove 
$
    \log \Lr{\tfrac{d}{\delta} \tfrac{\det(V_t)}{\det(V_1)}}
    \le 
    \tfrac{\beta}{8 \sigma^2 d}
$.
Indeed, in view of \cref{lem:boundnumofswitches} and the definition of $\beta$:
\begin{align*}
    \log \lr{\frac{d}{\delta} \frac{\det(V_t)}{\det(V_1)}} 
    &\le
    \log \frac{d}{\delta}
    +
    2 n \log T \\
    &\le
    4 n \log \frac{T}{\delta} 
    \tag{$T \ge d$} \\
    &\le
    \frac{\beta}{8 \sigma^2 d}~. \qedhere
\end{align*}
\end{proof}

\subsection{Proof of \cref{lem:relaxed-sdp-main}} 
\label{sec:relaxed-sdp-main-proof}

\begin{proof}
To prove the lemma, we aim to apply \cref{lem:sdpvaluelb,lem:optimsticsdppolicy}.

First, we show that $\mu \ge 1 + 2\vartheta \spnorm{V_t}$. 
Indeed, assuming $\spnorm{V_t} \le 4T$ and $T \ge \vartheta^{-2}$ implies that
\[
    1 + 2 \vartheta \spnorm{V_t} 
    \le
    \vartheta \sqrt{T} + 2 \vartheta \sqrt{4 T}
    \le
    5 \vartheta \sqrt{T}
    = 
    \mu~.
\]
Consequently, \cref{lem:sdpvaluelb} gives item (i) as well as that the dual solution of the SDP, $P_t$, is bounded as $\trnorm{P_t} \le \nu / \sigma^2$.

Next, note that $V_t \succeq \lambda I$ as well as
$
    \lambda
    \ge
    \nu \mu / \alpha_0 \sigma^2~,
$
where we have used the fact that 
$
    \nu \ge J^\st \ge P^\st \bullet W 
    \ge
    \qrmatrix \bullet W
    \ge
    \alpha_0 \sigma^2
$.
This gives $V_t \succeq (\nu \mu / \alpha_0 \sigma^2) I$. Thus we apply \cref{lem:optimsticsdppolicy} that shows item (ii).

To show item (iii), we have that $P_t$ is positive semi-definite immediately from the dual formulation of the SDP~\eqref{eq:relaxed-dual-sdp}.
Moreover, notice that \cref{lem:optimsticsdppolicy} also gives
\begin{align*}
    P_t
    =
    Q + K_t\tr P_t K_t + (A_t+B_t K_t)\tr P_t (A_t+B_t K_t) 
    - 
    \mu \trnorm{P_t} \IhatKt V_t^{-1} \IhatKt\tr
    ~,
\end{align*}
which we link with the true parameters $(\Ast \, \Bst)$ by combining the equation with \cref{lem:semidefinitebound}.
\end{proof}

\subsection{Proof of \cref{lem:stabilityandboundedness}} 
\label{sec:proofofsnb}

\begin{proof}
With probability at least $1-\delta/2$, \cref{lem:concentration} holds. Also, for any $t=1,\ldots,T$ with probability at least $1-\delta / 2T$, by the Hanson-Wright concentration inequality (\cref{thm:hanson}),
\[
    \norm{w_t} 
    \le 
    5\sigma \sqrt{d \log\frac{2T}{\delta}}
    \leq
    10 \sigma \sqrt{d \log\frac{T}{\delta}}
    ~,
\]
as $T \geq 2$.
Thus, via a union bound, both statements hold simultaneously with probability $1-\delta$.

Next, we show by induction on $t$ that $\trace(\Delta_t V_t \Delta_t\tr) \le 1$ and %$\mu \ge 1 + 2 \vartheta \spnorm{V_t}^{1/2}$
$\spnorm{V_t} \le 4T$. This will particularly ensure that the policies generated by \cref{alg:main} are sequentially strongly-stable which will give us $\norm{z_t}^2 \le 4 \kappa^4 e^{-\gamma(t-1)} \norm{x_1}^2 + \beta$ for all $t=1,\ldots,T$. 

For the base case, $t = 1$, we have by assumption
\begin{align*}
    \trace(\Delta_0 V_1 \Delta_0\tr)
    =
    \lambda \norm{\Delta_0}_F^2
    \le
    \frac{1}{4}
    \le 
    1~,
    \quad
    \spnorm{V_1}
    =
    \lambda
    \le 
    4 T~,
    \tag{$T \ge \tfrac{2^{30} \nu^{10} \vartheta^2}{\alpha_0^{10} \sigma^{20}} \implies T \ge \lambda$}
\end{align*}
and by definition of $\beta$:
\[
    \norm{z_1}^2 
    \le
    2 \kappa^2 \norm{x_1}^2
    \le 
    2 \beta T~.
\]

Now, assume that for all $s=1,\ldots,t-1$, $\trace(\Delta_s V_s \Delta_s\tr) \le 1$ and $\spnorm{V_s} \le 4T$. We show that $\trace(\Delta_t V_t \Delta_t\tr) \le 1$ and $\spnorm{V_t} \le 4T$. 

To that end we first show that $\norm{z_s}^2 \le 4 \kappa^4 e^{-(t-1)\gamma}\norm{x_1}^2 + \beta$ for all $s=1,\ldots,t$. Indeed, by $V_t \succeq \lambda I \succeq 16 \kappa^{10} \mu I$, \cref{lem:seqstrong-main} implies that policies generated by \cref{alg:main} up to round $t$ form a $(\kappa,\gamma)$-strongly stable sequence for $\kappa = \sqrt{2 \nu / \alpha_0 \sigma^2}$ and $\gamma = \thalf \kappa^{-2}$. Consequently, \cref{lem:strg-stab-norm} yields for all $s=1,\ldots,t$
\[
    \norm{x_t}
    \leq
    \kappa e^{-\gamma (t-1)/2} \norm{x_1} + \frac{20\kappa}{\gamma} \sigma \sqrt{d \log\frac{T}{\delta}}
    ~,
\]
which entails that
\begin{align*}
    \norm{z_s}^2
    &\le
    2 \kappa^2 \norm{x_s}^2 
    \tag{$z_s = \IhatKs x_s$, $\spnorm{\IhatKs}^2 \le 2 \kappa^2$} \\
    &\le
    2 \kappa^2 \lr{2 \kappa^2 e^{-\gamma(t-1)} \norm{x_{1}}^2 + \frac{800 \sigma^2 \kappa^2 d}{\gamma^2} \log\frac{T}{\delta}} \\
    &=
    4 \kappa^4 e^{-\gamma(t-1)} \norm{x_1}^2 + \beta~.
    \tag{$\gamma = \half \kappa^{-2}$, $\kappa = \sqrt{2 \nu / \alpha_0 \sigma^2}$}
\end{align*}
In particular, $\sum_{s=1}^t \norm{z_s}^2 \le 2 \beta T$ in view of \cref{lem:sumofnormssq}.
This, along with assuming $T \ge \lambda$, immediately gives
$$
    \spnorm{V_t}
    \le 
    \lambda + \beta^{-1} \sum_{s=1}^{t-1} \norm{z_s}^2
    \le
    4 T
    .
$$
Finally, as we've shown $\sum_{s=1}^t \norm{z_s}^2 \le 2 \beta T$, \cref{lem:concentration} additionally provides
$
    \trace(\Delta_t V_t \Delta_t\tr) 
    \le
    1
$.
\end{proof}

\subsection{Proof of \cref{lem:seqstrong-main}} 
\label{sec:seqstrong-main-proof}

\begin{proof}
The proof follows by applying \cref{thm:seqstrongstability} over the sequence $K_1,\ldots,K_T$ of policies generated by \cref{alg:main}.
To that end, define $\tau(t)$ as the last round before $t$ in which \cref{alg:main} updates its policy. Note that each policy $K_t$ is associated with $A_t, B_t, P_t$ and $V_{\tau(t)}$. 

Thus, to apply \cref{thm:seqstrongstability} it suffice to show that $\mu \ge 1 + 2 \vartheta \spnorm{V_t}^{1/2}$ and $V_t \succeq 16 \kappa^{10} \mu I$ for all rounds $t \ge 1$.
Indeed, as we assume $T \ge \vartheta^{-2}$ and $\spnorm{V_t} \le 4 T$, we have
\[
    1 + 2 \vartheta \spnorm{V_t} 
    \le
    \vartheta \sqrt{T} + 2 \vartheta \sqrt{4 T}
    \le
    5 \vartheta \sqrt{T}
    = 
    \mu~.
\]
Furthermore, using $\kappa = \sqrt{2 \nu / \alpha_0 \sigma^2}$, we have $V_t \succeq \lambda I$ where 
$
    \lambda
    \ge 
    2^9 \nu^5 \cdot 5 \vartheta \sqrt{T} / \alpha_0^5 \sigma^{10}
    =
    16\kappa^{10} \mu
$
as required.
\end{proof}

\subsection{Proof of \cref{lem:regretterm1}}
\begin{proof}
Let $N$ the last round $t$ such that $\calE_t$ holds.
Let $\tau_1 < \cdots < \tau_M$ be the time instances in which \cref{alg:main} changes policy up to round $N$, and let $\tau_0 = 1$, $\tau_{M+1} = N+1$.
By \cref{lem:boundnumofswitches}, as $\calE_N$ holds,
\[
    M 
    = 
    \lrfloor{ \log_2\frac{\det(V_{N})}{\det(V_1)} }
    \le
    2 n \log T~.
\]
Therefore,
\begin{align*}
    \sum_{t=1}^T \Lr{x_t\tr P_{t} x_t - x_{t+1}\tr P_{t} x_{t+1}} \ind{\calE_{t}}
    &=
    \sum_{t=1}^N \Lr{x_t\tr P_{t} x_t - x_{t+1}\tr P_{t} x_{t+1}} \\
    &=
    \sum_{i=0}^{M} x_{\tau_i}\tr P_{\tau_i} x_{\tau_i} - x_{\tau_{i+1}-1}\tr P_{\tau_i} x_{\tau_{i+1}-1}
    \\
    &\leq
    \sum_{i=0}^{M} x_{\tau_i}\tr P_{\tau_i} x_{\tau_i}
    ~.
\end{align*}
Since $\trnorm{P_t} \le \nu/\sigma^2$ and $\norm{x_t}^2 \le \norm{z_t}^2 \le 4 \kappa^4 \norm{x_1}^2 + \beta$ on $\calE_t$, we can bound
\begin{align*}
    \sum_{i=0}^{M} x_{\tau_i}\tr P_{\tau_i} x_{\tau_i}
    &\leq
    \sum_{i=0}^{M} \spnorm{P_{\tau_i}} \norm{x_{\tau_i}}^2 \\
    &\le
    (1+M) \cdot \frac{\nu}{\sigma^2} \cdot (32\kappa^8 \norm{x_1}^2 + \beta) \\
    &\le
    \frac{4 \nu (4\kappa^4 \norm{x_1}^2 + \beta)}{\sigma^2} n\log T
    ~,
\end{align*}
and the lemma follows.
\end{proof}

\subsection{Proof of \cref{lem:regretterm2}}

The lemma would follow directly from the following.

\begin{lemma} \label{lem:secondhpbound}
Let $\delta \in (0,1)$. 
Let $(\calF_t)_{t=1}^\infty$ be a filtration.
Let $w_1,w_2,\ldots \sim \calN(0,\sigma^2 I)$ be i.i.d Gaussian random variables. Let $v_1,v_2,\ldots$ be a sequence of vectors such that $v_t$ is $\calF_{t-1}$-measurable and $\sum_{t=1}^T \norm{v_t}^2 \le D^2$ almost surely for each $t$.
Then with probability $1-\delta$, 
\[
    \sum_{t=1}^{T} v_t\tr w_t 
    \leq
    2 \sigma D \sqrt{\log\frac{2}{\delta}}
    ~.
\]
\end{lemma}

\begin{proof}
Denote $Y_t = v_t\tr w_t$.
Note that, conditioned on the randomness before round $t$, each $Y_t$ is a zero-mean Gaussian random variable. Thus we can write $Y_t = \eta_t m_t$, where $m_t^2$ is the variance of $Y_t$ given $\calF_{t-1}$, and $\eta_t \sim \calN(0,1)$.

Let $\lambda > 0$. Using the observation above, we apply \cref{thm:regressionconcentration} with $V = \lambda$ to obtain that with probability $1-\delta$
\begin{equation}
    \label{eq:secondhpbound}
    \frac{\Lr{\sum_{t=1}^{T} \eta_t m_t}^2}{\lambda + \sum_{t=1}^{T} m_t^2}
    \le
    2 \log \frac{1 + \lambda^{-1} \sum_{t=1}^{T} m_t^2}{\delta}
    ~.
\end{equation}

We now proceed by upper bounding $\sum_{t=1}^T m_t^2$.
The variance of $Y_t$ given $\calF_{t-1}$ is: 
\[
    \sum_{t=1}^T m_t^2 
    = 
    \sigma^2 \sum_{t=1}^T \norm{v_t}^2
    \le
    \sigma^2 D^2~.
\]
Set $\lambda = \sigma^2 D^2$. Plugging the bound above into \cref{eq:secondhpbound} and rearranging gets us this lemma's statement.
\end{proof}

\begin{proof}[of \cref{lem:regretterm2}]
Apply \cref{lem:secondhpbound} with $v_t = P_{t} \Lr{\Ast \; \Bst} z_t \ind{\calE_t}$ and failure probability~$\tfrac{\delta}{2}$. Note that we have $\sum_{t=1}^T \norm{v_t}^2 \le \vartheta^2 (\nu/\sigma^2)^2 2 \beta T$. We obtain the bound
\[
    \sum_{t=1}^T w_t\tr P_{t} \Lr{\Ast + \Bst K_t} x_t \ind{\calE_t}
    \leq
    \frac{\nu \vartheta}{\sigma} \sqrt{3 \beta T \log\frac{4}{\delta}}
    ~.
\]
\end{proof}

\subsection{Proof of \cref{lem:regretterm3}}

The lemma is an immediate consequence of the following.

\begin{lemma} \label{lem:firsthpbound}
Let $(\calF_t)_{t=1}^\infty$ be a filtration, and let $M_1,M_2,\ldots$ be a sequence of symmetric positive semi-definite matrices such that $M_t$ is $\calF_{t-1}$-measurable and $\trnorm{M_t} \le D$ almost surely for each $t$. 
Further, let $w_1,w_2,\ldots \sim \calN(0,\sigma^2 I)$ be a sequence of i.i.d.~Gaussian random variables. 
Then for $T \ge 2$ and for any $\delta \in (0,1)$, it holds with probability at least $1-\delta$ that
\begin{align*}
    \sum_{t=1}^T \Lr{ w_t\tr M_t w_t - \sigma^2 \trnorm{M_t} }
    \leq
    8D \sigma^2 \sqrt{T \log^3\frac{4T}{\delta}}
    ~.
\end{align*}
\end{lemma}

\begin{proof}
Define the random variables $X_t = w_t\tr M_t w_t - \sigma^2 \trnorm{M_t}$ for all $t \geq 1$.
Observe that
\begin{align*}
    \E_t[X_t]
    =
    M_t \bullet \E_t[ w_t\tr w_t - \sigma^2 I ]
    =
    0
    .
\end{align*}
That is, $\set{X_t}_{t \ge 1}$ is a martingale difference sequence with respect to the filtration.
Moreover, $X_t \geq - \sigma^2 \trnorm{M_t} \geq -\sigma^2 D$ for all $t$ with probability one, 
However, $X_t$ is not bounded from above almost surely.
Therefore, consider the truncated random variables $\wt X_t = X_t \ind{X_t \le \Gamma}$ with threshold $\Gamma = 5D\sigma^2 \log(4T/\delta)$.
By Azuma's inequality (\cref{thm:azuma}), we have with probability at least $1-\delta/2$ that
\begin{align*}
    \sum_{t=1}^T \wt X_t - \sum_{t=1}^T \E_t[\wt X_t]
    \leq
    \Gamma \sqrt{2T \log\frac{1}{\delta}}
    \leq
    8\sigma^2 D\sqrt{T \log^3\frac{4T}{\delta}}
    .
\end{align*}

On the other hand, let us show that $\wt X_t = X_t$ for all $t$ with probability at least $1-\delta/2$.
Indeed, note that $\E_t[w_t\tr M_t w_t] = \sigma^2 \trnorm{M_t}$, and using the Hanson-Wright inequality (\cref{thm:hanson}), for any fixed $t$ and any $\delta' \in (0,1/e)$ we have
\begin{align} \label{eq:w-tail}
    w_t\tr M_t w_t
    \leq 
    \sigma^2 \trnorm{M_t} + 4\sigma^2 \norm{M_t^{1/2}} \norm{M_t^{1/2}}_\frob \log(1/\delta')
    \leq 
    5 \sigma^2 \trnorm{M_t} \log(1/\delta')
\end{align}
with probability at least $1-\delta'$.
Since $\trnorm{M_t} \le D$, this implies that $X_t \leq \Gamma$ for all $t$ with probability at least $1-\delta/2T$, which in turn means that $\wt X_t = X_t$ for all $t$ with the same probability.

Finally, since $\sum_{t=1}^T \E_t[\wt X_t] \leq \sum_{t=1}^T \E_t[X_t] = 0$ (as $\wt X_t \leq X_t$ for all $t$), we obtain that with probability at least $1-\delta$,
\begin{align*}
    \sum_{t=1}^T \Lr{ w_t\tr M_t w_t - \sigma^2 \trnorm{M_t} } 
    =
    \sum_{t=1}^T X_t
    =
    \sum_{t=1}^T \wt X_t
    \leq
    8\sigma^2 D\sqrt{T \log^3\frac{4T}{\delta}}
    .
    &\qedhere
\end{align*}
\end{proof}

\begin{proof}[of \cref{lem:regretterm3}]
Apply \cref{lem:firsthpbound} with failure probability $\delta/2$ and define $M_t = P_{t} \ind{\calE_t}$, and note that $\trnorm{P_{t}} \le \nu/\sigma^2$ on $\calE_t$. 
\end{proof}

\subsection{Proof of \cref{lem:regretterm4}}

\begin{proof}
Note that for any $t$ we have on $\calE_t$ that
\[
    \frac{1}{\beta} z_t V_t^{-1} z_t 
    \le 
    \frac{1}{\beta} z_t V_1^{-1} z_t
    =
    \frac{1}{\beta \lambda} \norm{z_t}^2
    \le
    \frac{1 + \norm{x_1}^2}{\vartheta \sqrt{T}}
    \le 1~,
\]
using $\lambda \ge \vartheta \sqrt{T}$ as $\kappa \ge 1$, and $T \ge \vartheta^{-2} (1+\norm{x_1}^2)^2$.
Let $N$ the last round $t$ such that $\calE_t$ holds, then
\begin{align*}
    \sum_{t=1}^T z_t \tr V_t^{-1} z_t \ind{\calE_t}
    &=
    \beta \sum_{t=1}^N \frac{1}{\beta} z_t \tr V_t^{-1} z_t \\
    &\le 
    \beta \sum_{t=1}^N 2 \log\frac{\det(V_{t+1})}{\det(V_t)} 
    \tag{\cref{lem:det1}} \\
    &= 
    2 \beta \log \frac{\det(V_{N+1})}{\det(V_1)} \\
    &\le
    4 \beta n \log T
    ~.
    \tag{\cref{lem:sumofnormssq,lem:boundnumofswitches}}
\end{align*}
\end{proof}

% \subsection{Proof of \cref{lem:rtildebound}} \label{sec:proofofrtbound}

% \begin{proof}
% Let us bound each of the terms in \cref{lem:tildertbound}.

% For \cref{eq:regretterm1}, 

% Next, to upper bound \cref{eq:regretterm2} we 

% To upper bound \cref{eq:regretterm3} we 
% We now apply a union bound to show that both \cref{eq:regretterm2,eq:regretterm3} hold with probability at least $1-\delta$.

% Let us turn to \cref{eq:regretterm5}. 

% The final bound on $\wt R_T$ is given by plugging the upper bounds above in each of the terms of \cref{lem:tildertbound} and simplifying.
% \end{proof}

\section{Proofs of \cref{sec:sdp-analysis}}

\subsection{Proof of \cref{lem:sigmabound}} \label{sec:sigmaboundproof}

\begin{proof}
    We have
    \[
        \spnorm{(X+\Delta) \Sigma (X+\Delta)\tr - X \Sigma X\tr} 
        \le
        \Sigma \bullet \Lr{(X+\Delta)\tr (X+\Delta) - X\tr X}~,
    \]
    and by \cref{lem:semidefinitebound},
    \[
        (X+\Delta)\tr (X+\Delta) - X\tr X 
        \preceq 
        \mu V^{-1}~. \qedhere
    \]
\end{proof}

\subsection{Proof of \cref{lem:strongstability}} \label{sec:proofofstrongstab}

\begin{proof}
\cref{lem:optimsticsdppolicy,lem:semidefinitebound} imply
\[
    P_t 
    \succeq 
    Q + K_t\tr R K_t + \Lr{\Ast+\Bst K_t}\tr P_t \Lr{\Ast+\Bst K_t} - 2 \mu \trnorm{P_t} \IKt\tr V^{-1} \IKt
    .
\]
Now, recall that, by assumption, 
$
    V 
    \succeq 
    2 \kappa^2 \mu I
$.
Therefore, by \cref{lem:sigmabound},
\begin{align*}
    \mu \trnorm{P_t} V^{-1}
    \preceq 
    \mu \cdot \frac{\nu}{\sigma^2} \cdot \frac{1}{2 \kappa^2 \mu} I
    = 
    \frac{\nu}{2 \sigma^2 \kappa^2} I
    =
    \frac{\alpha_0}{4} I
    ~.
\end{align*}
Hence, as $Q \succeq \alpha_0 I$ and $R \succeq \alpha_0 I$,
\begin{align}
    P_t 
    &\succeq 
    \thalf \alpha_0 I + \thalf \alpha_0 K_t\tr K_t + \Lr{\Ast + \Bst K_t}\tr P_t \Lr{\Ast + \Bst K_t} \label{eq:plowerbound} \\
    &\succeq 
    \thalf \alpha_0 I + \Lr{\Ast + \Bst K_t}\tr P_t \Lr{\Ast + \Bst K_t} \nonumber
    ~.
\end{align}
In particular, this shows that $P_t \succeq \half \alpha_0 I$.
Further, using again the fact that $\trnorm{P_t} \le \nu / \sigma^2$ (\cref{lem:sdpvaluelb}) to bound $P_t - \half \alpha_0 I \preceq (1-\kappa^{-2}) P_t$ and rearranging yields
\[
    P_t^{-1/2} \Lr{\Ast + \Bst K_t}\tr P_t \Lr{\Ast + \Bst K_t} P_t^{-1/2}
    \preceq
    (1 - \kappa^{-2}) I
    ~.
\]
Letting $H_t = P_t^{1/2}$ and $L_t = P_t^{-1/2} \Lr{\Ast + \Bst K_t} P_t^{1/2}$, we have established that $\spnorm{L_t} \leq \sqrt{1-\kappa^{-2}} \le 1 - \half \kappa^{-2}$, as well as $\spnorm{H_t} \leq \sqrt{\nu/\sigma^2}$ and $\spnorm{H_t^{-1}} \leq \sqrt{2/\alpha_0}$.
To bound the norm of $K_t$, observe that \cref{eq:plowerbound} implies $P_t \succeq \thalf \alpha_0 K_t\tr K_t$ hence
\[
    \spnorm{K_t} 
    \leq 
    \sqrt{\frac{2}{\alpha_0} \spnorm{P}} 
    \le 
    \sqrt{\frac{2\nu}{\alpha_0 \sigma^2}} 
    = 
    \kappa
    ~.
\]
As $\Ast+ \Bst K_t = H_t L_t H_t^{-1}$, this shows that $K_t$ is $(\kappa,\gamma)$-strongly stable.
\end{proof}

\subsection{Proof of \cref{lem:phatclosetopstar}} \label{sec:proofofpclosetopst}

\begin{proof}
It suffices to show that $P_t \preceq P^\st \preceq P_t +  \tfrac{\alpha_0 \gamma}{2} I$ for all $t \ge 1$.

For $P_t \preceq P^\st$, let $\Kst$ denote the optimal policy corresponding to $P^\st$.
As $P^\st$ is the solution to the Riccati equation: 
\begin{align*}
    P^\st
    =
    Q + \Kst\tr R \Kst + (\Ast+B\Kst)\tr P^\st (\Ast+B\Kst)
    .
    \tag{\cref{eq:ricatti}}
\end{align*}
On the other hand, applying \cref{lem:semidefinitebound} over \cref{eq:relaxed-dual-sdp} gives
\[
    \begin{psmallmatrix} Q - P_t & 0 \\ 0 & R \end{psmallmatrix} + \Lr{ \Ast \;  \Bst}\tr P_t \Lr{ \Ast \;  \Bst} \succeq 0
    ~,
\]
which particularly implies
\[
    P_t
    \preceq 
    Q + \Kst\tr R \Kst + (\Ast+\Bst \Kst)\tr P_t (\Ast+ \Bst \Kst)
    .  
\]
Subtracting the two inequalities gets us
\[
    P_t - P^\st
    \preceq 
    \Lr{\Ast + \Bst \Kst}\tr \Lr{ P - P^\st} \Lr{\Ast + \Bst \Kst}
    ~,
\]
and, as $\Kst$ is a (strongly) stable policy, \cref{lem:lypstonglystable} implies $P - P^\st \preceq 0$. 

For the converse inequality, \cref{eq:bellman} implies
\begin{align*}
    P^\st 
    \preceq 
    Q + K\tr R K + \Lr{\Ast+\Bst K}\tr P^\st \Lr{\Ast+\Bst K}
    .
\end{align*}
On the other hand, combining \cref{lem:optimsticsdppolicy,lem:semidefinitebound} yields
\[
    P_t 
    \succeq 
    Q + K\tr R K + \Lr{\Ast + \Bst K}\tr P_t \Lr{\Ast + \Bst K} - 2 \mu \trnorm{P_t} \IhatK\tr V_t^{-1} \IhatK
    .
\]
Subtracting the two matrix inequalities gets us
\[
    P^\st - P_t
    \preceq  
    \Lr{\Ast + \Bst K}\tr \Lr{P^\st - P_t} \Lr{\Ast + \Bst K} + 2 \mu \trnorm{P_t} \IhatK\tr V_t^{-1} \IhatK
    ~.
\]
Applying \cref{lem:lypstonglystable} shows
\[
    P^\st - P_t 
    \preceq 
    \frac{2 \kappa^2 \mu}{\gamma} \trnorm{P_t} \spnorm{\IhatK\tr V_t^{-1} \IhatK} I
    ~.
\]
Moreover, $\spnorm{K} \leq \kappa$ provides
$
    \spnorm{\IhatK}^2
    \le
    1 + \kappa^2
    \le
    2 \kappa^2
$,
thus
$
    \spnorm{\IhatK\tr V_t^{-1} \IhatK} 
    \le 
    2 \kappa^2 \spnorm{V_t^{-1}}
$
.
Finally, by \cref{lem:sdpvaluelb} and the lower bound on $V_t$,
\[
    \frac{4 \kappa^4}{\gamma} \trnorm{P_t} \spnorm{V_t^{-1}}
    \le
    \frac{4 \kappa^4}{\gamma} 
    \cdot 
    \frac{\nu}{\sigma^2}
    \cdot
    \frac{1}{16 \kappa^{10} \mu}
    =
    \frac{\alpha_0 \gamma}{2}
\]
where we have used $\kappa = \sqrt{2 \nu / \sigma^2 \alpha_0}$ and $\gamma = \half \kappa^{-2}$.
\end{proof}

\section{Proofs of \cref{sec:warmstart}}

\subsection{Proof of \cref{thm:Vws}} 
\label{sec:warmstartproof}
% 
% To prove the theorem, we need the following lemma.
We first require the following lemma.

\begin{lemma} \label{lem:xbound}
Assume $x_1 = 0$. Let $\delta \in (0,1/e)$.
With probability at least $1-\delta$, for all $t = 1,\ldots,T_{\ws}+1$ it holds that 
\[
    \norm{x_t} 
    \leq 
    \frac{4 \sigma \kappa_{\ws}}{\gamma_{\ws}} \sqrt{\Lr{d + k \vartheta^2 \kappa_{\ws}^2} \log\frac{T_{\ws}}{\delta}}
    ~.
\]
\end{lemma}

\begin{proof}
We begin by upper bounding the norm of $x_t$ using the strong stability of $K_{\ws}$. Let $u_t = K_{\ws} x_t + \eta_t$ where $\eta_t \sim \calN(0, 2 \sigma^2 \kappa_{\ws}^2 I)$. We have,
\[
    x_{t+1} = (\Ast+\Bst K_{\ws}) x_t + \Bst \eta_t + w_t~,
\]
and, as $\eta_t$ is independent of $x_t$, we can think about the state transitions as if they are done according the another LQR system that is exactly the same as the original one except that the noise term is now $\Bst \eta_t + w_t$ instead of $w_t$. Thus, applying \cref{lem:strg-stab-norm}:
\[
    \norm{x_t} 
    \le 
    \frac{\kappa_{\ws}}{\gamma_{\ws}} \max_{0\le s \le t-1} \norm{\Bst \eta_s + w_s}~.
\]

Next, $\Bst \eta_s + w_s$ is a Gaussian random variable with zero mean and covariance $C = 2 \sigma^2 \kappa_{\ws}^2 \Bst \Bst\tr + \sigma^2 I$.
Using the Hanson-Wright inequality (\cref{thm:hanson}) and a union bound, with probability $1-\delta$, for all $t = 1,\ldots,T_{\ws}+1$, 
\begin{align*}
    \norm{\Bst \eta_t + w_t}^2 
    &\leq 
    5 \trace(C) \log(T_{\ws}/\delta) 
    \\
    &= 
    5 \sigma^2 \Lr{d + 2 \kappa_{\ws}^2 \norm{\Bst}_F^2} \log(T_{\ws}/\delta)
    \\
    &\leq
    10 \sigma^2 \Lr{d + k \kappa_{\ws}^2 \vartheta^2} \log(T_{\ws}/\delta)
    ~. 
    \qedhere
\end{align*}
\end{proof}

% Next, we show that $V$ is lower bounded with high probability.
% %
% \begin{theorem}
% \label{thm:warmstart}
% Let $\delta \in (0,1)$. Suppose that $T_{\ws} \ge 400 (d+k + \log(1/\delta))$, then with probability at least $1-\delta$, 
% \[
% \lambda_{\min}(V) \ge \frac{T_{\ws} \sigma^2}{80}~.
% \]
% \end{theorem}

% In the remainder of this section we will prove the theorem.

For the lower bound, we also require the next lemma.

\begin{lemma} \label{lemma:singlevectorconcentration}
Let $\delta \in (0,1)$, and let $n \in \bbR^n$ be any unit vector. 
Suppose that $T_{\ws} \ge 200 \log(1/\delta)$. 
Then with probability at least $1-\delta$ we have $n\tr V n \ge T_{\ws} \cdot \sigma^2/40$.
\end{lemma}

The proof relies on a couple of technical results.
In what follows, we let $(\calF_t)_{t=1}^\infty$ be the filtration with respect to which $\{w_t,u_t\}_{t=1}^\infty$ is adapted.

\begin{lemma} \label{lemma:varlb}
% Denote $z_t = \begin{psmallmatrix} x_t \\ u_t \end{psmallmatrix}$. 
%
For all $t$ we have $\E_t \big[ z_t z_t\tr \mid \mathcal{F}_{t-1}] \succeq (\sigma^2/2) I$.
\end{lemma}

\begin{proof}
Note that since $W = \sigma^2 I$ we have $\E [x_t x_t\tr \mid \mathcal{F}_{t-1}] \succeq \sigma^2 I$ for each $t \ge 1$, and so 
\begin{align*}
    \E \big[ z_t z_t\tr \bigm| \mathcal{F}_{t-1} \big] 
    &= \begin{psmallmatrix} I \\ K_{\ws} \end{psmallmatrix} \E \big[ x_t x_t\tr \bigm| \mathcal{F}_t \big] \begin{psmallmatrix} I \\ K_{\ws} \end{psmallmatrix}\tr + \begin{psmallmatrix} 0 & 0 \\ 0 & 2 \sigma^2 \kappa_{\ws}^2 I \end{psmallmatrix} \\
    &\succeq \sigma^2 \begin{psmallmatrix} I & K_{\ws}\tr \\ K_{\ws} & K_{\ws} K_{\ws}\tr + 2 \kappa_{\ws}^2 I \end{psmallmatrix} \\
    &\succeq \sigma^2 \begin{psmallmatrix} I & K_{\ws}\tr \\ K_{\ws} & \half I + 2K_{\ws} K_{\ws}\tr \end{psmallmatrix} \tag{$\norm{K_{\ws}} \le \kappa_{\ws},\; \kappa_{\ws} \ge 1$} \\
    &= \frac{\sigma^2}{2} I + \sigma^2 \begin{psmallmatrix} \sqrt{2}^{-1} I \\ \sqrt{2} K_{\ws} \end{psmallmatrix} \begin{psmallmatrix} \sqrt{2}^{-1} I \\ \sqrt{2} K_{\ws} \end{psmallmatrix}\tr \\
    &\succeq \frac{\sigma^2}{2} I
    ~.
\end{align*}
The lemma now follows by taking expectations.
\end{proof}

\begin{lemma} \label{lemma:probabilitylb}
Denote $S_t = n\tr z_t$, and let $E_t$ be an indicator random variable that equals 1 if $S_t^2 > \sigma^2 / 4$ and 0 otherwise. 
Then $\E[E_t \bigm| \mathcal{F}_{t-1}] \ge 1/5$.
\end{lemma}

\begin{proof}
We have,
\begin{align*}
    \Pr[E_t = 1 \bigm| \mathcal{F}_{t-1}] &= \Pr[S_t^2 > \sigma^2 / 4  \bigm| \mathcal{F}_{t-1}] \\
    &= \Pr[|S_t| > \sigma / 2 \bigm| \mathcal{F}_{t-1}] \\
    &\ge \Pr[S_t - \E S_t > \sigma / 2 \bigm| \mathcal{F}_{t-1}] \tag{$S_t - \E S_t$ is a symmetric r.v.} \\
    &= \Pr[\sqrt{\text{Var}[S_t | \calF_{t-1}]} Z > \sigma / 2] \tag{$Z$ is a standard Gaussian r.v.} \\
    &\ge \Pr[\sqrt{\sigma^2 / 2} Z > \sigma / 2]  \tag{\cref{lemma:varlb}} \\
    &= \Pr[Z > 1/\sqrt{2}] \\
    &\ge 2^{-3/2} e^{-(1/\sqrt{2})^2} \tag{Standard Gaussian tail lower bound} \\
    &\ge 1/5
    ~.
    &\qedhere
\end{align*}
\end{proof}

\begin{proof}[of \cref{lemma:singlevectorconcentration}]
Let $U_t = E_t - \E_t[E_t \mid \calF_{t-1}]$. Then $U_t$ is a martingale difference sequence with $|U_t| \le 1$ almost surely. Applying Azuma's inequality, we have that with probability at least $1-\delta$,
\[
\sum_{t=1}^{T_{\ws}} U_t \ge -\sqrt{2 T_{\ws} \log \frac{1}{\delta}} \ge -\frac{T_{\ws}}{10}~, \tag{$T_{\ws} \ge 200 \log(1/\delta)$}
\]
which means that, by \cref{lemma:probabilitylb},
\begin{align*}
\sum_{t=1}^{T_{\ws}} E_t &\ge \sum_{t=1}^{T_{\ws}} \E_t[E_t] - \frac{T_{\ws}}{10} 
\ge \sum_{t=1}^{T_{\ws}} \frac{1}{5} - \frac{T_{\ws}}{10}
= \frac{T_{\ws}}{10}~. 
\end{align*}

Now, by definition of $E_t$, $S_t^2 \ge E_t \cdot \sigma^2/4$. Therefore, with probability at least $1-\delta$,
\[
    n\tr V n 
    = 
    \sum_{t=1}^{T_{\ws}} S_t^2 \ge \sum_{t=1}^{T_{\ws}} E_t \cdot \sigma^2/4 
    = 
    \frac{T_{\ws} \sigma^2}{40}
    ~.\qedhere
\]
\end{proof}

We are now ready to prove the main theorem of this section.

\begin{proof}[of \cref{thm:Vws}]
We first prove the upper bound.
Let $u_t = K_{\ws} x_t + \eta_t$ where $\eta_t \sim \calN(0, 2 \sigma^2 \kappa_{\ws}^2 I)$. 
Then,
\[
\trace(V_{\ws}) = \sum_{t=1}^{T_{\ws}} \norm{z_t}^2~,
\]
and, as $\norm{K_{\ws}} \le \kappa_{\ws}$ and $\kappa_{\ws} \ge 1$:
\[
    \norm{z_t} 
    \le 
    \norm{x_t} + \norm{u_t} 
    \le 
    \norm{x_t} + \norm{K_{\ws} x_t} + \norm{\eta_t} 
    \le 
    2 \kappa_{\ws} \norm{x_t} + \norm{\eta_t}~.
\]
Now, using a union bound, with probability $1-\delta/2$ we have for all $t=1,\ldots,T_{\ws}+1$ by \cref{lem:xbound}
\[
    \norm{x_t} 
    \le 
    \frac{4 \sigma \kappa_{\ws}}{\gamma_{\ws}} \sqrt{\Lr{d + k \vartheta^2 \kappa_{\ws}^2} \log(4T_{\ws}/\delta)}~,
\]
and by the Hanson-Wright inequality $\norm{\eta_t}^2 \le 10 \sigma^2 \kappa_{\ws}^2 k \log(4T_{\ws}/\delta)$ for all $t$. 
Therefore,
\begin{align*}
    \norm{z_t} 
    &\le 
    2 \kappa_{\ws} \cdot \frac{4 \sigma \kappa_{\ws}}{\gamma_{\ws}} \sqrt{\Lr{d + k \vartheta^2 \kappa_{\ws}^2} \log(4T_{\ws}/\delta)} 
    + 
    4 \kappa_{\ws} \sigma \sqrt{k \log(4 T_{\ws}/\delta)} \\
    &\le
    \frac{12 \sigma \kappa_{\ws}^2}{\gamma_{\ws}} \sqrt{\Lr{n + k \vartheta^2 \kappa_{\ws}^2} \log(4T_{\ws}/\delta)}
    ~.
\end{align*}

We next turn to lower bounding the smallest eigenvalue of $V$; we will actually prove that $\spnorm{V^{-1}} \le 80/(T_{\ws} \sigma^2)$.
Let $\calN(1/4)$ be a minimal $1/4$-net of $\bbS^{n-1}$, and define the set $M = \{V^{-1/2} u / \norm{V^{-1/2} u} \; : \; u \in \calN(1/4) \}$.
Suppose that $T_{\ws} \ge 200\log(|M|/\delta)$. 
Applying a union bound, we get that with probability at least $1-\delta$ simultaneously for all $n \in M$
\[
n\tr V n \ge \frac{T_{\ws} \sigma^2}{40}~. \tag{\cref{lemma:singlevectorconcentration}}
\]
Using the definition of $M$, this entails that for all $u \in \calN(1/4)$
\begin{equation}
\label{eq:vinvqf}
    u\tr V^{-1} u \le \frac{40}{T_{\ws} \sigma^2}~.
\end{equation}

Next, let $z$ be the eigenvector corresponding to the minimum eigenvalue of $V$, and let $u_z \in \calN(1/4)$ be such that $\|z-u_z\|\le 1/4$. Then,
\begin{align*}
    \spnorm{V^{-1}} &= z\tr V^{-1} z \\
    &\le u_z\tr V^{-1} u_z + (z-u_z)\tr V^{-1} (z+u_z) \\
    &\le u_z\tr V^{-1} u_z + \norm{z-u_z} \spnorm{V^{-1}} \Lr{\norm{z}+\norm{u_z}} \\
    &\le \frac{40}{T_{\ws} \sigma^2} + \frac{1}{4} \spnorm{V^{-1}} \cdot 2~. \tag{\cref{eq:vinvqf}. $z$ and $u_z$ are unit vectors}
\end{align*}
Rearranging gets us
$
\spnorm{V^{-1}} \le 80 / (T_{\ws} \sigma^2)
$
as required.
Note that $\lvert M \rvert = \lvert \calN(1/4) \rvert$, and by standard bounds on the size of $\epsilon$-nets, $\lvert \calN(1/4) \rvert \le 12^n$. That is, for $T_{\ws}$ to be larger that $200 \log(\lvert M \rvert/\delta)$ it suffices to have $T_{\ws} \ge 400 (n + \log(1/\delta))$.

To show that a bound on the estimation error of $\Lr{A_{\ws} \, B_{\ws}}$, set $V = V_{\ws} + \sigma^2 \vartheta^{-2} I$, and
\[
    \Lr{A_{\ws} \; B_{\ws}} 
    =  
    \lr{\sum_{t=1}^{T_{\ws}} x_{t+1} z_t\tr} V^{-1}~.
\]
Applying \cref{lem:concentration} with these parameters and $\Lr{A_0 \; B_0} = 0$, shows that with probability $1-\delta/2$
\begin{align*}
    \trace(\Delta_{\ws} V \Delta_{\ws}\tr) 
    &\le 
    4 \sigma^2 d \log \lr{\frac{d}{\delta} \det \lr{I + \vartheta^2 \sigma^{-2} V_{\ws}}} 
    +
    2 \sigma^2 \vartheta^{-2} \norm{\lr{\Ast \; \Bst}}_F^2 \\
    &\le
    4 \sigma^2 d \log \frac{d}{\delta}
    +
    4 \sigma^2 d n \log \biggl( 1 + T_{\ws} \cdot \frac{300 \vartheta^2 \kappa_{\ws}^4}{\gamma_{\ws}^2} (1+\vartheta^2 \kappa_{\ws}^2) \log \frac{T_0}{\delta} \biggr) 
    + 
    2 \sigma^2 d \\
    &\le
    20 n^2 \sigma^2 \log(T_{\ws}/\delta)~,
\end{align*}
using $\log \det X \le n \log (\trace(X) / n)$ for a positive-definite $X \in \bbR^{n \times n}$, by our choice of $T_{\ws}$ and the lower bound on $T$.
\end{proof}

\end{document}